%% file: main.tex
\documentclass[10pt]{article}
\usepackage[margin=1in]{geometry}
\usepackage[round,compress]{natbib}
\usepackage{parskip}
\usepackage{algorithm}
\usepackage{subcaption}
\usepackage{tianyu}

\input{notations}

\title{Collaborative Heterogeneous Causal Inference Beyond Meta-analysis}
\author{
Tianyu Guo\footnotemark[1]\thanks{Department of Statistics, UC Berkeley. Email: \texttt{tianyu\_guo@berkeley.edu}}
\and
Sai Praneeth Karimireddy\footnotemark[2]\thanks{Department of EECS, UC Berkeley.}
\and
Michael I. Jordan\footnotemark[1] \footnotemark[2]
}
\begin{document}

\maketitle

\input{Sections/abstract}
\input{Sections/Introduction}

\input{Sections/Problem_Setup}

\input{Sections/asymptotic-dist}

\input{Sections/off-policy}

\input{Sections/experiments}

\input{Sections/discussion}

\clearpage

\bibliography{references}
\newpage
\appendix
\onecolumn
\makeatletter
\def\renewtheorem#1{%
  \expandafter\let\csname#1\endcsname\relax
  \expandafter\let\csname c@#1\endcsname\relax
  \gdef\renewtheorem@envname{#1}
  \renewtheorem@secpar
}
\def\renewtheorem@secpar{\@ifnextchar[{\renewtheorem@numberedlike}{\renewtheorem@nonumberedlike}}
\def\renewtheorem@numberedlike[#1]#2{\newtheorem{\renewtheorem@envname}[#1]{#2}}
\def\renewtheorem@nonumberedlike#1{  
\def\renewtheorem@caption{#1}
\edef\renewtheorem@nowithin{\noexpand\newtheorem{\renewtheorem@envname}{\renewtheorem@caption}}
\renewtheorem@thirdpar
}
\def\renewtheorem@thirdpar{\@ifnextchar[{\renewtheorem@within}{\renewtheorem@nowithin}}
\def\renewtheorem@within[#1]{\renewtheorem@nowithin[#1]}
\makeatother

\renewtheorem{theorem}{Theorem}[section]
\renewtheorem{lemma}[theorem]{Lemma}
\renewtheorem{remark}{Remark}
\renewtheorem{corollary}[theorem]{Corollary}
\renewtheorem{corollary*}{Corollary}
\renewtheorem{observation}[theorem]{Observation}
\renewtheorem{proposition}[theorem]{Proposition}
\renewtheorem{definition}[theorem]{Definition}
\renewtheorem{claim}[theorem]{Claim}
\renewtheorem{fact}[theorem]{Fact}
\renewtheorem{assumption}{Assumption}%
\renewcommand{\theassumption}{\Alph{assumption}}
\renewtheorem{conjecture}[theorem]{Conjecture}
\include{Sections/appendix-thm-proof}

\include{Sections/appendix-exp}

\end{document}

%% file: notations.tex
\def\bI{{\mathbf I}}

\def\bx{{\mathbf x}}

\newcommand{\client}{k}
\newcommand{\indiv}{i}
\newcommand{\fed}{\textsc{Clb}}

\newcommand{\clientsup}{{(\client)}}
\newcommand{\clientzsup}{{(\client,\treatmentdata)}}
\newcommand{\clienttrtsup}{{(\client,\groupt)}}
\newcommand{\clientctrsup}{{(\client,\groupc)}}

\newcommand{\clientp}{r}

\newcommand{\pooled}{\textrm{Meta}}
\newcommand{\groupt}{1}
\newcommand{\groupc}{0}
\newcommand{\hajek}{\text{Hájek}}
\newcommand{\HT}{\text{HT}}

\newcommand{\clientset}{\mc{S}}
\newcommand{\clientnum}{{K}}
\newcommand{\clientdataset}{\mc{D}}
\newcommand{\dataset}{\mc{D}}
\newcommand{\fl}{\textit{FL}}

\newcommand{\outcome}{Y}
\newcommand{\covariates}{X}
\newcommand{\treatment}{Z}

\newcommand{\covariatesdata}{\bx}
\newcommand{\treatmentdata}{z}

\newcommand{\taudr}{\delta}

\newcommand{\modeloutcome}{m}
\newcommand{\modeltreated}{m_1}
\newcommand{\modelcontrol}{m_0}
\newcommand{\modelprop}{e}

\newcommand{\loss}{\ell}

\newcommand{\dr}{\textsc{aipw}}
\newcommand{\tarsrcratio}{\lambda}
\newcommand{\boundy}{M_\outcome}

\newcommand{\datasize}{N}
\newcommand{\datasizeall}{\datasize_\clientset}

\newcommand{\dimension}{d}

\newcommand{\ratio}{w}

\newcommand{\pslower}{c}

\newcommand{\sumtrtweights}{\sum_{r=1}^\clientnum \modelprop^{(r,1)}\paren{\covariates_\indiv}}
\newcommand{\sumctrweights}{\sum_{r=1}^\clientnum \modelprop^{(r,0)}\paren{\covariates_\indiv}}
\newcommand{\sumesttrtweights}{\sum_{r=1}^\clientnum \hat\modelprop^{(r,1)}\paren{\covariates_\indiv}}
\newcommand{\sumestctrweights}{\sum_{r=1}^\clientnum \hat\modelprop^{(r,0)}\paren{\covariates_\indiv}}
\newcommand{\sumrvtrtweights}{\sum_{r=1}^\clientnum \modelprop^{(r,1)}\paren{\covariates}}
\newcommand{\sumrvctrweights}{\sum_{r=1}^\clientnum \modelprop^{(r,0)}\paren{\covariates}}

\newcommand{\select}{S}
\newcommand{\alloct}{A}

\newcommand{\sumsite}{G}

\newcommand{\selectctrsup}{{\paren{\client,\groupc}}}
\newcommand{\selecttrtsup}{{\paren{\client,\groupt}}}

\newcommand{\ate}{\textsc{ate}}

\newcommand{\ipw}{\textsc{ipw}}
\newcommand{\poolipw}{\pooled-\textsc{ipw}}
\newcommand{\collabipw}{\fed-\textsc{ipw}}

\newcommand{\collab}{\fed}

\newcommand{\target}{\textrm{t}}
\newcommand{\targetsup}{{\paren{\textrm{t}}}}
\newcommand{\datasettarget}{\dataset^{\paren{\target}}}

\newcommand{\ratepower}{\xi}
\newcommand{\dml}{\textsc{aipw}}

\newcommand{\sumclient}{\sum_{\client=1}^\clientnum}

%% file: Sections/abstract.tex
\begin{abstract}
Collaboration between different data centers is often challenged by heterogeneity across sites. To account for the heterogeneity, the state-of-the-art method is to re-weight the covariate distributions in each site to match the distribution of the target population. Nevertheless, this method could easily fail when a certain site couldn't cover the entire population. Moreover, it still relies on the concept of traditional meta-analysis after adjusting for the distribution shift.

In this work, we propose a collaborative inverse propensity score weighting estimator for causal inference with heterogeneous data. Instead of adjusting the distribution shift separately, we use weighted propensity score models to collaboratively adjust for the distribution shift. Our method shows significant improvements over the methods based on meta-analysis when heterogeneity increases. To account for the vulnerable density estimation, we further discuss the double machine method and show the possibility of using nonparametric density estimation with $\dimension<8$ and a flexible machine learning method to guarantee asymptotic normality. We propose a federated learning algorithm to collaboratively train the outcome model while preserving privacy. Using synthetic and real datasets, we demonstrate the advantages of our method.
\end{abstract}

%% file: Sections/Introduction.tex
\section{Introduction}
The booming of Federated Learning (\fl) has drawn attention in medical and social sciences, where sharing datasets between data centers is often limited. However, their research focuses more on Causal Inference, in which prediction gets less attention, whereas valid inference is the main focus. For example, Meta-analysis takes the weighted mean of published estimators of the average treatment effect (\ate) and mainly focuses on choosing optimal weights and making inferences.

Given homogeneous data, how could Federated Learning help Causal Inference? The estimation of \ate~commonly incorporates \textit{nuisance prediction models}, e.g., the propensity score model.
Thanks to the homogeneity, we can use \fl~methods to train a shared propensity score model, then each site gets its own \ate~estimator, and finally, the central server uses Meta-analysis to take the weighted mean. 

Nevertheless, given heterogeneous data, Federated Learning seems to play a negligible role. Since propensity score models differ between sites, training a shared model is meaningless. As a result, all methods fall within the scope of Meta-analysis. For example, to estimate the \ate~for a target site, \citet{han_federated_2022} and \citet{han_multiply_2023} consider using density ratio to re-weight source sites and summarizing estimates from source sites with Meta-analysis.

We propose a novel method tailored for collaboration with heterogeneous data.
Suppose we have $\clientnum$ sites, denote the $\ate$ as $\tau$, the nuisance propensity model as $\modelprop$, the site-wise weight as $\eta_\client$. Instead of taking the weighted mean afterward, we directly take the weighted mean of nuisance models and get $\hat\tau_\client(\textstyle \sum_{r=1}^\clientnum \eta_r \hat{\modelprop}_r)$ in each site $\client$, which is inconsistent. Then,  we could recover a consistent estimator $\hat\tau_{\fed}$ by taking the average across all sites. Equations \eqref{eqn:old-est} and \eqref{eqn:new-est} summarize the previous and our estimators.
\begin{align}
&\hat\tau_{\text{homo}} = \sumclient \eta_\client \hat\tau_\client(\hat{\modelprop}_{\fl})
\quad
\hat\tau_{\text{heter}} = \sumclient \eta_\client \hat\tau_\client(\hat{\modelprop}_\client),
\label{eqn:old-est}\\
&\text{we propose:}\quad
\hat\tau_{\fed} = \sumclient \hat\tau_\client(\textstyle \sum_{r=1}^\clientnum \eta_r \hat{\modelprop}_r).\label{eqn:new-est}
\end{align}

Our method outperforms previous ones in several ways: first, it is the first method that allows collaboration across disjoint domains without additional assumptions; second, it achieves better accuracy than Meta-analysis; third, it remains stable even as the heterogeneity between sites increases, which encourages collaboration from a broader range. We provide theory and experiment to demonstrate these claims.

%% file: Sections/Problem_Setup.tex
\section{Problem Setup}
\label{sec:setup}
We use  $\clientset = [\clientnum]$ to denote the set of sites, with $\dataset^\clientsup$ being the dataset of site $\client$. Let $\treatment$ be the binary treatment, $\covariates \in \R^\dimension$ be the covariates with dimension $\dimension$, $\outcome$ be the outcome. Let $\outcome(\treatmentdata)$ be the potential outcome under treatment $\treatmentdata\in\set{\groupt,\groupc}$. 
Classical causal inference only copes with the biased sampling of $\treatment$. However, we need to cope with multiple sites. We first present a motivating example from Meta-analysis to model the actual data-generating procedure. 
\begin{example}[Collaboration of Clinical Trails]\label{exp:rct}\citet{koesters_agomelatine_2013} reviews clinical trials of Agomelatine, an antidepressant drug approved by the European Medicines Agency in 2009.
The 13 included trials have different data sizes and demographic distributions. One study was carried out on individuals aged 60 or above, and the remaining is for all ages. Each study reports the mean difference in Hamilton Rating Scale for Depression (HRSD) scores between treatment and control groups. 
\end{example}

The target group in Example \ref{exp:rct} is the patients with depression. However, each clinical trial is a biased sample from the target population. Abstracting from this example, we propose a \textit{sampling-selecting} framework for collaborative causal inference:

\begin{enumerate}
\item \textit{Sampling}: 
Sample an individual $\indiv$ from the target distribution, let $\select_{\indiv} \in \set{\paren{\client,\treatment}\mid \client \in \clientset, \treatment \in \set{0,1}} \cup \emptyset$ be the selection indicator. If $\select = \paren{\client,1}$, individual $\indiv$ gets selected to site $\client$ and gets treated. If $\select = \emptyset$, the individual is eliminated from the dataset. 
Sample $(\covariates_\indiv,\outcome_\indiv(1),\outcome_\indiv(0),\select_{\indiv})$ i.i.d. from the target distribution according to Equation \eqref{eqn:def-prop} and get the pooled dataset \eqref{eqn:pooled-dataset}.
\begin{align}
\textstyle
& \P\paren{\select = \emptyset\mid \covariates}  = \modelprop^\emptyset\paren{\covariates},\nonumber\\
& \P\paren{\select = \paren{\client,\treatmentdata}\mid \covariates} = \modelprop^\clientzsup\paren{\covariates}\nonumber\\
& \text{ with } \modelprop^\emptyset(\covariates) + \sumclient \sum_{\treatmentdata \in\set{0,1}} \modelprop^\clientzsup(\covariates) = 1.\label{eqn:def-prop}\\
&\clientdataset_\pooled = \set{\paren{\covariates_\indiv,\outcome_\indiv(1),\outcome_\indiv(0),\select_\indiv} \mid \indiv \in [\datasize]}.\label{eqn:pooled-dataset}
\end{align}

\item \textit{Selecting:}
We use $\treatment\paren{\select_\indiv}$ to denote the treatment indicator corresponding to $\select_\indiv$. We follow the potential outcome framework and invoke the  Stable Unit Treatment Value Assumption (SUTVA). Therefore, $\outcome_\indiv = \treatment\paren{\select_\indiv} \outcome_\indiv(\groupt) + \set{1 - \treatment\paren{\select_\indiv}} \outcome_\indiv(\groupc) $.
Split $\clientdataset_\pooled$ to each site and treatment/control groups according to $\select$ and get
\begin{equation}\label{eqn:site-dataset}
\clientdataset^\clientsup = \set{\paren{\covariates_\indiv,\outcome_\indiv,\treatment_\indiv} \in \dataset_\pooled \mid \select_\indiv = \paren{\client,\treatment_\indiv}}.
\end{equation}
\end{enumerate}
Furthermore, census data commonly reflects the target distribution of covariates. Therefore, we assume there's a public dataset $\datasettarget$ that contains covariates information.
\begin{equation}
\dataset^\targetsup = \set{\paren{\covariates_\indiv}\mid \covariates_\indiv \text{ drawn $i.i.d.$ from target distribution}}.
\end{equation}
\begin{figure}
    \centering
    \begin{subfigure}{0.45\textwidth}
    \centering
    \caption{\textit{Sampling-Selecting} Framework.}
    \label{fig:sampling-then-selecting}
    \includegraphics[width=\linewidth]{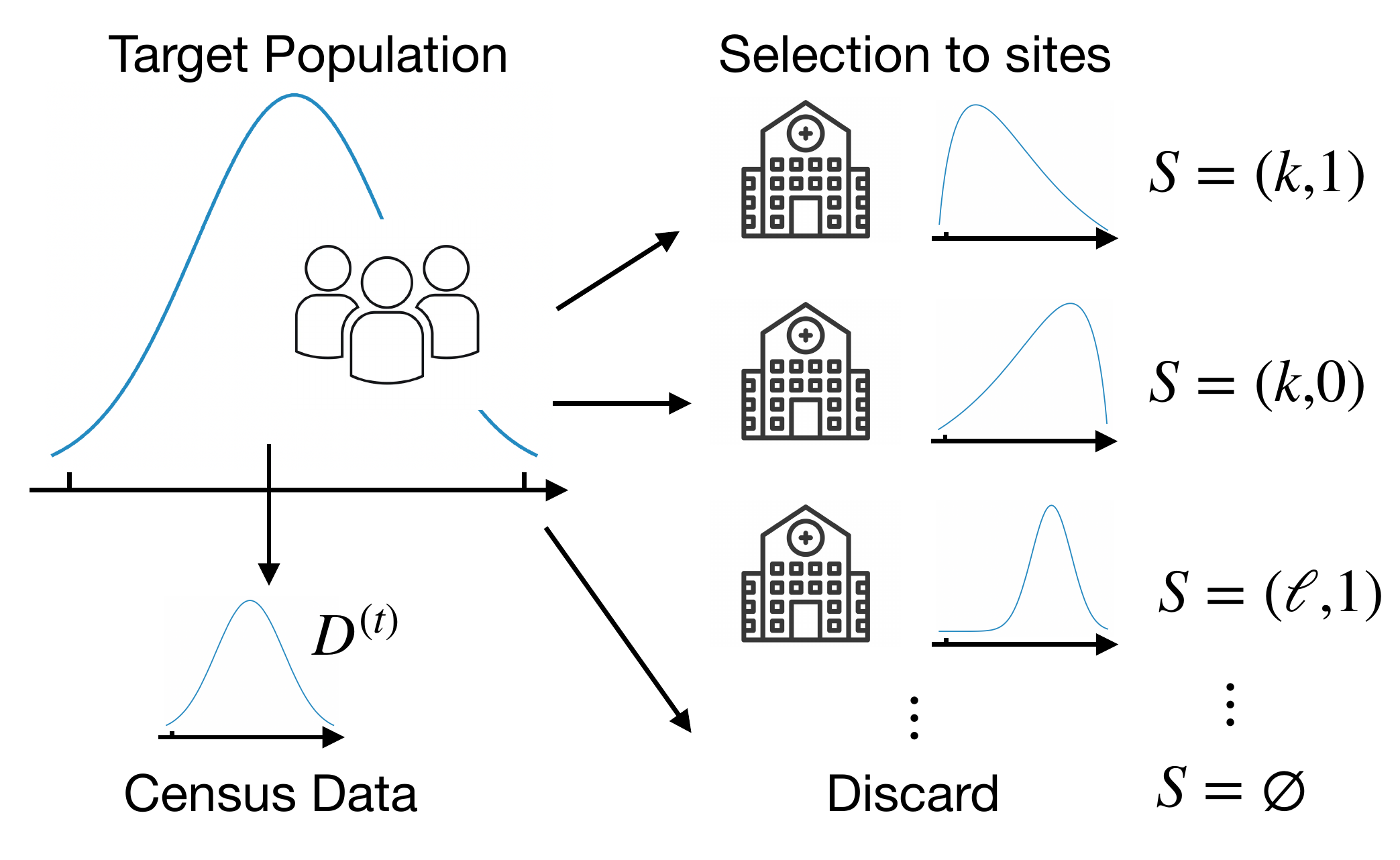}
\end{subfigure}
\hfill
\begin{subfigure}{0.45\textwidth}
    \centering
    \caption{KL-MSE curve of different estimators}
    \label{fig:kl-mse}
    \includegraphics[width=\linewidth]{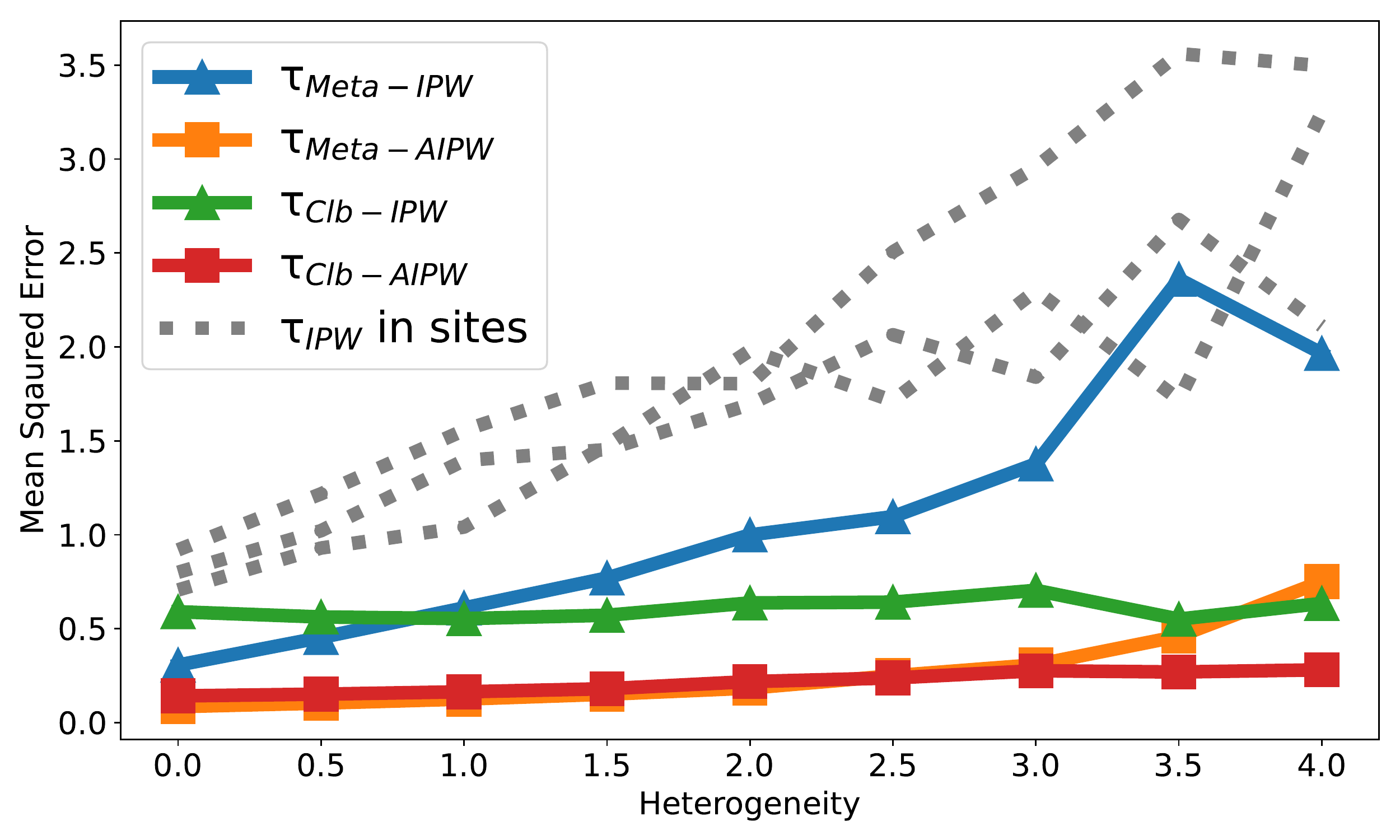}
\end{subfigure}
    \caption{Visualization of the data-generating process and the comparison of proposed estimators.}
    \label{fig:intro}
\end{figure}
Figure \ref{fig:sampling-then-selecting} visualizes the data-generating process. Each site selects from the target distribution in a different way. The selection indicator $\select = \paren{\client,z}$ describes the sampling mechanism of the $\treatment = z$ group in site $\client$. The figure shows distribution shifts with left and right tilting ($\select = \paren{\client,1}, \paren{\client,0}$), under coverage ($\select = \paren{\ell,1}$), and discarded data $\select=\emptyset$. \citet{xiong_federated_2022} assumes that i.e., $\P\paren{\select = \emptyset} = 0$ and consider the pooled dataset. In contrast, 
we allow $\select = \emptyset$ to reflect the biased sampling of the pooled dataset $\cup_{\client=1}^\clientnum \dataset^\clientsup$ from the target population, which is neglected by \citet{xiong_federated_2022}. For example, if all sites include fewer men in the dataset, we would have that $\P\paren{\select \neq \emptyset\mid \text{men}} < \P\paren{\select \neq \emptyset\mid \text{other genders}}$.

One may question that there's no real \textit{sampling-selecting} process since each site collects data independently. 
A possible answer is to recall the well-accepted quasi-experiment framework \citep{cook_experimental_2002}. For example, to understand the effect of gender on an outcome. The quasi-experiment framework imagines that individual $\covariates_\indiv$ firstly gets $i.i.d.$ sampled, then gets ``treated'' by gender $G_\indiv$, although there's no actual ``gender'' assignment process.
The \textit{sampling-selecting} framework extends the quasi-experiment to multiple-site settings.

The objective is to estimate the average causal effect on the target distribution $\tau = \E\brac{\outcome(1)-\outcome(0)}$. 
A foundation for identifying $\tau$ is Assumption \ref{ass:unconfoundedness}.
\begin{assumption}[Homogeneity and unconfoundedness]
\label{ass:unconfoundedness}
We have that 
\begin{equation}\label{eqn:homo-cate}
\paren{\outcome(1),\outcome(0)} \indep \select  \mid \covariates 
\end{equation}
\end{assumption}
More than unconfounded treatment assignment, Assumption \ref{ass:unconfoundedness} also implies that the individual treatment effects are the same across sites. In Example \ref{exp:rct}, when fixing $X$ for an individual $\indiv$, if the effect of Agomelatine still varies across sites, collaboration is meaningless due to unmeasured confounders. Violation of Assumption \ref{ass:unconfoundedness} is sometimes termed as anti-causal learning \citep{farahani_brief_2020}.

Another foundation for identifying the causal effect is the overlapping assumption. There are two kinds of overlapping assumptions. Given an individual, Assumption \ref{ass:indiv-overlap} requires each site to select them with non-zero probability, whereas Assumption \ref{ass:overall-overlap} only requires the overall selection probability to be non-zero.
\begin{assumption}[Individual-Overlapping]
\label{ass:indiv-overlap}
We have that
$$\min_{\covariatesdata,\client} \set{\P\paren{\treatment = 1, \select = 1 \mid \covariatesdata, \alloct = \client}} > \pslower > 0. $$
\end{assumption}
\begin{assumption}[Overall-Overlapping]
\label{ass:overall-overlap}
We have that $$\min_{\covariatesdata}\set{\P\paren{\treatment = 1, \select = 1 \mid \covariatesdata}} > \pslower > 0.$$
\end{assumption}

Revisiting Example \ref{exp:rct},  Assumption \ref{ass:unconfoundedness} is guaranteed by the experimental design and by the similar effect of drugs given sufficient demographic information. Assumption \ref{ass:indiv-overlap} fails since one site only includes senior patients. While Assumption \ref{ass:overall-overlap} holds since other sites collect data from all ages.

We provide a counter-example showing that Assumptions \ref{ass:unconfoundedness} and \ref{ass:overall-overlap} might not hold.

\begin{example}[Collaboration of Observational Studies with Unmeasured Confounder]\label{exp:obs}
\citet{betthauser_systematic_2023} reviews observational studies regarding the learning deficits of school-aged children during COVID-19. Among 42 included observational studies, four were from middle-income countries, and the remaining were from high-income countries. Over half of the studies didn't collect covariates and took the difference in means of grades before and after the pandemic. 

Since over half of the studies didn't collect covariates, there are unmeasured confounders. Therefore, Assumption \ref{ass:unconfoundedness} is unlikely to hold. Moreover, if the target distribution is school-aged children from the entire world, both \ref{ass:indiv-overlap} and \ref{ass:overall-overlap} fail since low-income countries are missing in the study. We suggest avoiding collaboration in this case.

We have some additional notations: Denote $\E\brac{\outcome_1}$ as $\mu_1$ and $\E\brac{\outcome_0}$ as $\mu_0$. Define $\datasize_\clientset = \sumclient \datasize^\clientsup$ with $\datasize^\clientsup = \abs{\dataset^\clientsup}$ being the sample size of dataset $\client$. Note that $\datasize_\clientset < \datasize$ since we drop the individuals with $\select = \emptyset$.

\end{example}

\subsection{Related Work}
There are extensive attempts in Meta-analysis literature to cope with heterogeneity \citep{borenstein_introduction_2007,borenstein_basic_2010,higgins_re-evaluation_2009}. For example, by assuming that the average treatment effect follows the normal distribution across sites, many propose using random effects models \citep{hedges_fixed-and_1998,riley_interpretation_2011} instead of fixed effects models \citep{tufanaru_fixed_2015}. There are other ways, such as using site-specific information and conduct Meta-regression \citep{van_houwelingen_advanced_2002, glynn_introduction_2010}, using quasi-likelihood \citep{tufanaru_fixed_2015}. More recently, \citet{cheng_adaptive_2021} propose a penalized method for integrating heterogeneous causal effects. However, all methods need strong parametric assumptions on the heterogeneity. It's still necessary to rely on qualitative understandings of heterogeneity based on summary statistics \citep{stroup_meta-analysis_2000}.

Causal Inference literature also has a growing interest in collaboration with concerns in external validity \citep{concato_randomized_2000,rothwell_external_2005,colnet_causal_2023}.
\cite{yang_combining_2020} propose a Rao-Blackwellization method for incorporating RCT and observational studies with unmeasured confounders to improve the estimation efficiency. Recently, more works try to incorporate federated learning in causal inference \citep{xiong_federated_2022,han_privacy-preserving_2023,guo_robust_2023,vo_federated_2023}. \citep{vo_adaptive_2022} proposes adaptive kernel methods under the causal graph model. Several focus on inference. For example, \citep{xiong_federated_2022,hu_collaborative_2022} assumes homogeneous models and proposes a collaboration framework that avoids direct data merging. \citet{han_federated_2022,han_multiply_2023} considers heterogeneous sample selection under parametric distribution shift assumptions. Nevertheless, most new methods still fall under the framework of Meta-analysis.

As a broader interest, our work also uses double machine learning.\cite{chernozhukov_doubledebiased_2018, athey_machine_2019}. It extends the doubly robust estimator \citep{bang_doubly_2005, glynn_introduction_2010,funk_doubly_2011} to non-parametric and machine learning methods  \citep{huang_correcting_2006,sugiyama_direct_2007,sugiyama_covariate_2007,wager_estimation_2017,tibshirani_regression_1996}. We adopt it in particular to mitigate the hardness of estimating density ratio \citep{farahani_brief_2020,hardle_nonparametric_2004}.

%% file: Sections/asymptotic-dist.tex
\section{Collaborative Inverse Propensity Score Weighting}
\label{sec:ipw}
The inverse propensity score weighting (\ipw) estimator plays a central role in causal inference. We generalize it to collaborative setting, thinking $\modelprop^\clientzsup\paren{\covariates} = \P\paren{\select = \clientzsup\mid\covariates}$ as a generalized version of propensity score. We begin with using the oracle propensity score models and then discuss how to estimate the models.
\subsection{The \fed-IPW~estimator}
As a benchmark, consider the method where each site calculates its own \ipw~estimator for \ate~and takes weighted sum, which is the standard method in Meta-analysis. Since we assume propensity score models are correct, it's not necessary to use $L_1$ penalty as in \citet{han_federated_2022}. Define
\begin{align}
& \hat\tau_{\pooled} = \sumclient \eta^\clientsup \paren{\hat{\mu}^\clientsup_{\pooled,\groupt} - \hat{\mu}^\clientsup_{\pooled,\groupc}},\quad \text{with} \nonumber \\
& 
\hat{\mu}^\clientsup_{\pooled,\groupt} = \frac{1}{\hat\datasize^\clientsup_{\pooled,\groupt}}\sum_{\indiv\in\clientdataset^\clientsup} \frac{\treatment_\indiv \outcome_\indiv}{\modelprop^\clienttrtsup\paren{\covariates_\indiv}}, \label{eqn:meta-ipw-t}\\
&\hat{\mu}^\clientsup_{\pooled,\groupc} = \frac{1}{\hat\datasize^\clientsup_{\pooled,\groupc}}\sum_{\indiv\in\clientdataset^\clientsup} \frac{\paren{1 - \treatment_\indiv} \outcome_\indiv}{\modelprop^\clientctrsup\paren{\covariates_\indiv}}, \label{eqn:meta-ipw-c}\\
&\hat{\datasize}^\clientsup_{\pooled,\groupt} = \sum_{\indiv\in\clientdataset^\clientsup} \frac{\treatment_\indiv }{\modelprop^\clienttrtsup\paren{\covariates_\indiv}}, \text{ and} \label{eqn:meta-ipw-n-1}\\
&\hat{\datasize}^\clientsup_{\pooled,\groupc} = \sum_{\indiv\in\clientdataset^\clientsup} \frac{1 - \treatment_\indiv }{\modelprop^\clientctrsup\paren{\covariates_\indiv}}. \label{eqn:meta-ipw-n-0}
\end{align}
Equations \eqref{eqn:meta-ipw-t}, \eqref{eqn:meta-ipw-c}, \eqref{eqn:meta-ipw-n-1}, and \eqref{eqn:meta-ipw-n-0} take the \hajek~form \cite{little_statistical_2019}, in which we use a consistent estimator $\hat\datasize$ for the sample size. It always achieves better numerical stability and smaller variance than directly using sample size. More importantly, as we will show later, we could only identify $\modelprop^\clientzsup\paren{\covariates}$ up to a constant factor, \hajek~form releases us from the identifiability issue.

The best choice of $\eta^\clientsup$ is the inverse variance. In specific, denoting $\Var\paren{\hat\tau^\clientsup_\pooled} = \paren{\sigma^\clientsup_\pooled}^2 $, the optimal weights are 
$
\eta^\clientsup \propto \paren{\sigma^\clientsup_\pooled}^{-2}.
$
See \citet{cheng_adaptive_2021, ye_debiased_2021} for more discussions.

\poolipw~is designed for review studies \citep{borenstein_introduction_2007} rather than collaboration. Each site must be able to obtain a valid estimator. But, a single site would commonly suffer from under-coverage of the entire population. Revisiting Example \ref{exp:rct}, one site only takes experiments for older people. Due to their under-coverage, they can never get the valid \ate~estimator for the entire population, so it's impossible to incorporate them into the \poolipw~estimator.

Alternatively, we introduce the \collabipw~estimator. \collabipw~directly takes the weighted mean of heterogeneous propensity score functions. In specific, to estimate $\mu_1$, we use
\begin{align}\label{eqn:collab-ipw-motivation}
\hat\mu_{\collab,\groupt}^\clientsup = \frac{1}{\hat\datasize_{\collab,1}^\clientsup}\sum_{\indiv\in\clientdataset^\clientsup} \frac{\eta^\clientsup \treatment_\indiv \outcome_\indiv}{\sum_{r=1}^\clientnum \eta^{(r)}\modelprop^{(r,1)}\paren{\covariates_\indiv}}, \\
\text{ with } \hat\datasize^\clientsup_{\fed,\groupt} = \sum_{\indiv\in\clientdataset^\clientsup} \frac{\eta^\clientsup \treatment_\indiv}{\sum_{r=1}^\clientnum \eta^{(r)}\modelprop^{(r,1)}\paren{\covariates_\indiv}}\nonumber
\end{align}
We have that
\[
\E\brac{\hat\mu^\clientsup_{\collab,\groupt}} = \E\bracBig{\frac{\eta^\clientsup \modelprop^\clienttrtsup\paren{\covariates}\outcome_1}{\sum_{r=1}^\clientnum \eta^{(r)}\modelprop^{(r,1)}\paren{\covariates}}},
\]
which means that it's not consistent for $\mu_1$. However, when we take summation of $\hat\mu^\clientsup_{\fed,\groupt}$ across $\client$, we get that
\begin{align*}
\E\brac{\hat\mu_{\fed,\groupt}} = \E\bracBig{\frac{\sumclient \eta^\clientsup \modelprop^\clienttrtsup\paren{\covariates}\outcome_1}{\sum_{r=1}^\clientnum \eta^{(r)}\modelprop^{(r,1)}\paren{\covariates}}} = \mu_1.
\end{align*}
It allows collaboration between disjoint domains. In Example \ref{exp:rct}, the site that only includes elders could compute $\hat\mu_{\collab,\groupt}^\clientsup$ without worrying about their under-coverage. Given a young patient $\covariates$ from other sites. We have that $\modelprop^\clienttrtsup\paren{\covariates} = 0$ but $ \modelprop^{(r,1)}\paren{\covariates} > 0$ for $r\neq \client$, which ensures a non-zero denominator. The estimators for $\mu_0$ follow the same manner, which we relegate to the appendix. We could compute $\hat\tau_{\fed}$ in a fully federated way, as presented in Algorithm \ref{alg:fed}.

\begin{algorithm}
\caption{\collabipw~Algorithm}
\label{alg:fed}
\begin{algorithmic}[1]
\REQUIRE $K$ datasets with $\clientdataset^\clientsup$ as shown in Equation \eqref{eqn:site-dataset}. Each site publishes their propensity score models $\modelprop^\clienttrtsup\paren{\covariates}$ and $\modelprop^\clientctrsup\paren{\covariates}$.
\FOR{$k = 1$ to $K$}
    \STATE At site $\client$, calculate $\hat\mu^\clientsup_{\collab,\groupt}$ , $\hat\mu^\clientsup_{\collab,\groupc}$,  $\hat\datasize_{\collab,\groupt}^\clientsup$, and $\hat\datasize_{\collab,\groupt}^\clientsup$ according to Equation \eqref{eqn:collab-ipw-motivation}. Send them to the central server.
\ENDFOR
\STATE Central server computes 
\begin{equation}
    \hat\tau_{\fed} = \hat\mu_{\fed,\groupt} - \hat\mu_{\fed,\groupc},
\end{equation}
where $\hat\mu_{\fed,\groupt}$ is the average of $\hat{\mu}_{\fed,\groupt}^\clientsup$ weighted by $\hat\datasize_{\fed,\groupt}^\clientsup$, with $\hat\mu_{\fed,\groupc}$ following the same manner.
\end{algorithmic}
\end{algorithm}

The best choice of $\eta^\clientsup$ is data-dependent and thus could not be obtained from one round of communication. Therefore, we suggest taking vanilla weights $\eta^\clientsup = 1$ for all $\client$. Notice that
\begin{equation}
\sumclient\modelprop^\clienttrtsup\paren{\covariates} = \P\paren{\treatment\paren{\select} = 1\mid\covariates},
\end{equation}
which means that the vanilla weights match the propensity score for $\treatment$ in the pooled dataset. More importantly, we find that the vanilla weights would already make the \collabipw~estimator uniformly better than \poolipw~estimator.

\begin{proposition}[\poolipw~Estimator]\label{prop:asymp-pooled}
Given Assumptions \ref{ass:unconfoundedness} and \ref{ass:indiv-overlap}, using inverse variance weighting, as $\datasize \to \infty$, we have that
\begin{align*}
\sqrt{\datasize}(\hat\tau_{\pooled} - \tau)\converge{d}\normal\paren{0,v^2_{\pooled}},
\end{align*}
where
\begin{align*}
v^2_{\pooled} = \setBig{ \sumclient \E\bracBig{\frac{\paren{\outcome_\groupt - \mu_1}^2}{\modelprop^\selecttrtsup\paren{\covariates}} + \frac{\paren{\outcome_\groupc - \mu_0}^2}{\modelprop^\selectctrsup\paren{\covariates}}}}^\inv.
\end{align*}
\end{proposition}
\begin{theorem}[\collabipw~Estimator]\label{thm:fed-hajek} 
Given Assumptions \ref{ass:unconfoundedness} and \ref{ass:overall-overlap}, using vanilla weights for \collabipw, as $\datasize \to \infty$, we have that 
\begin{equation}
\sqrt{\datasize}(\hat\tau_{\fed} - \tau) \converge{d} \normal(0,v^2_\collab), \label{eqn:asymp-known-ps}
\end{equation}
where
\begin{align*}
v^2_\collab & = \E\bracBig{\frac{\paren{\outcome_1 - \mu_1}^2}{\sumclient  \modelprop^\clientsup\paren{\covariates}} + \frac{\paren{\outcome_0 - \mu_0}^2}{\sumrvctrweights}} 
\end{align*}
Moreover, we have that \[v^2_\fed \leq v^2_\pooled.\]
\end{theorem}
There are two ways to understand why $\hat\tau_\fed$ is better: First, \poolipw~takes the weighted mean site-wise, whereas \collab-\ipw~takes the weighted mean individual-wise. Given each individual $\covariates_\indiv$, \collab-\ipw~adaptively puts more weights on sites with larger $\modelprop^\clientsup\paren{\covariates_\indiv}$. Whereas \pooled-\ipw~uses the same weights for any $\covariates_\indiv$. Second, \collabipw~utilizes coarser \textit{balancing scores} \cite{imbens_causal_2015}. \textit{Balancing score} is a generalization of the propensity score. Any function of covariates is sufficient for adjusting the confoundingness between $\treatment$ and $\outcome$. 
The \poolipw~uses $\P\paren{\select\mid\covariates}$ as its inverse weights, and \collabipw~uses $\P\paren{\treatment\paren{\select}\mid\covariates}$. Theorem \ref{thm:balancing} shows that they are both balancing scores.
\begin{theorem}\label{thm:balancing}
We have that
\begin{align*}
&\paren{\outcome(1),\outcome(0)} \indep \treatment\paren{\select} \mid  \P\paren{\select\mid\covariates},\text{ and}\\
& \paren{\outcome(1),\outcome(0)} \indep \treatment\paren{\select} \mid \P\paren{\treatment\paren{\select}\mid\covariates}.
\end{align*}
\end{theorem}
Notice that $\P\paren{\select\mid\covariates}$ has an auxiliary variable $\client\paren{\covariates}$ comparing to $\P\paren{\treatment\paren{\select}\mid\covariates}$. But $\client\paren{\covariates}$ is superfluous since it doesn't affect $\paren{\outcome_1,\outcome_0}$. As a result, \collabipw~gets better efficiency by maintaining a smaller model. 
A simpler model benefits us by maintaining fewer variables to adjust for, thus attaining better efficiency.
Similar ideas occur extensively in model selection literature \cite{raschka_model_2020}.

\subsection{Estimation of propensity score models}
We start from the identification of $\modelprop^\clientzsup\paren{\covariates}$. Since we have no information on the dropped set $\clientdataset_\emptyset = \set{\indiv\mid\select_\indiv = \emptyset}$, it's impossible to identify all parameters. For instance, multiplying $\datasize$ by a factor $2$ and dividing $\modelprop^\clientzsup\paren{\covariates}$ by $2$ would lead to the same observed distribution. However, identifiability is guaranteed up to a constant factor. And thanks to the \hajek~forms of our \ipw~estimators, identification up to a constant is enough.
\begin{proposition}
\label{prop:identification}
We have that
\begin{align}
\modelprop^\clientzsup\paren{\covariates}=r^\clientzsup\paren{\covariates}  \P\paren{\select = \clientzsup\mid \select\neq \emptyset} \P\paren{\select = \emptyset}\nonumber.\label{eqn:identifiability-modelprop}
\end{align}
where $r^\clientsup\paren{\covariates} = {p(\covariates\mid \select = \clientzsup) }/{p(\covariates)}$ is the density ratio function, which is identifiable. Meanwhile, $\P\paren{\select = \clientzsup\mid \select\neq \emptyset}$ is identifiable by taking $\datasize^\clientsup/\datasizeall$. Only $\P\paren{\select = \emptyset}$ is not identifiable.
\end{proposition}
We focus on estimating density ratio $r^\clientzsup\paren{\covariates}$.
We suggest two methods from the large literature on density ratio estimation. \citet{han_federated_2022} applies a parametric exponential tilting model. They assumes that $r^\clientzsup\paren{\covariates} = \exp\paren{\psi\paren{\covariates}^\top \gamma^{\clientzsup}}$ for a given representation function $\psi$ (such as $\psi(\covariatesdata) = \covariatesdata$) and unknown parameter $\gamma^{\clientzsup}$. We could estimate $\gamma$ through the method of moments, i.e., finding $\hat\gamma^\clientzsup$ that solves
\begin{align*}
 &~\sum_{\indiv\in\clientdataset^\clientsup} \treatment_\indiv \psi\paren{\covariates_\indiv} \exp\paren{\psi\paren{\covariates}^\top \gamma^{\clientzsup}} \\
 = &~\sum_{\indiv\in\dataset^\target} \psi\paren{\covariates_\indiv} \exp\paren{\psi\paren{\covariates}^\top \gamma^{\clientzsup}},
\end{align*}
which is equivalent to entropy balancing \cite{zhao_entropy_2017}. Recently, motivated by Matching \cite{abadie_matching_2016} and K-Nearest Neighbour \cite{zhang_efficient_2018}, \citet{lin_estimation_2021} propose a minimax nonparametric way to estimate the density ratio. Using their method, we have that 
\begin{equation}
\hat{r}^\clientzsup\paren{\covariatesdata} = \frac{\datasize^\targetsup}{\sum_{\indiv\in\dataset^\clientsup}\treatment_\indiv} \frac{M}{W\paren{\covariatesdata;\dataset^\target,\dataset^\clientzsup}},
\end{equation}
where $W\paren{\covariatesdata;\dataset^\target,\dataset^\clientzsup}$ means the total number of units in $\dataset^\target$ that $\covariatesdata$ is close to $\covariates_\indiv$ than its $M$-nearest neighbour in $\dataset^\clientzsup$. See \citet{lin_estimation_2021} for more detail.
We have the following convergence rates for them
\begin{proposition}[Point-wise error of density estimation]\label{prop:est-acc}
Given $\covariatesdata\in\R^d$, if the exponential tilting model is correctly specified, we have that
\begin{equation}
\E\bracBig{\abs{\exp\paren{\psi\paren{\covariatesdata}^\top \hat\gamma^\clientzsup} - r^\clientzsup\paren{\covariatesdata}}} = O(N^{-1/2}).
\end{equation}
For the nonparametric method, we have that
\begin{equation}
\E\bracBig{\abs{\hat{r}^\clientzsup\paren{\covariatesdata} - r^\clientzsup\paren{\covariatesdata}}} = O(N^{-1/(2+\dimension)}).
\end{equation}
\end{proposition}

%% file: Sections/off-policy.tex
\section{Incorporating Outcome Models}
\label{sec:dml}
Density ratio estimation is challenging and can easily fail under mis-specification or due to the curse of dimensionality. Therefore, it is essential to incorporate outcome models to mitigate the errors caused by density ratio estimation. To maintain consistent structure with Section \ref{sec:ipw}, we first discuss how to incorporate outcome models in the estimator and then discuss how to learn the outcome models.
\subsection{Decoupled AIPW estimator}
The augmented inverse propensity score weighted (\dml) estimator \cite{bang_doubly_2005} employs Neyman orthogonality to construct an asymptotically normal estimator even if nuisance models converge at slower rates. We introduce their idea to the collaboration setting.

How to use outcome models? Due to the biased selection of $\select$, directly taking the mean across all source data renders the estimator inconsistent. A natural idea is to use the inverse propensity score to adjust the distribution and get that
\begin{align*}
\hat\tau_{\text{adjust}}  = \frac{1}{\datasize} \sum_{\select_\indiv\neq\emptyset} \bracBig{\frac{\hat\modeltreated\paren{\covariates_\indiv}}{\hat\modelprop^\clienttrtsup\paren{\covariates_\indiv}} - \frac{\hat\modelcontrol\paren{\covariates_\indiv}}{\hat\modelprop^\clientctrsup\paren{\covariates_\indiv}}}.
\end{align*}
This is the choice of \citet{han_federated_2022}.
However, the consistency of $\hat\tau_{\text{adjust}}$ substantially depends on the density ratio function, making the regression model useless. Alternatively, we make use of the public census dataset $\clientdataset^\targetsup$. As discussed in Section \ref{sec:setup}, $\clientdataset^\targetsup$ provides public information for $\covariates$ in the target distribution. Utilizing it, we propose a \textit{decoupled} \dml~estimator.
\begin{equation}\label{eqn:def of dml}
\hat\tau_{\dml} = \frac{1}{\datasize^\targetsup} \sum_{\indiv=1}^{\datasize^\targetsup} \bracBig{\hat\modeltreated\paren{\covariates_\indiv^\targetsup} - \hat\modelcontrol\paren{\covariates_\indiv^\targetsup}} + \sumclient \hat\delta^\clientsup_{\dml},
\end{equation}
with $\hat\delta^\clientsup_{\dml}$ having two versions:
\begin{align*}
&\hat\delta^\clientsup_{\pooled-\dml} = \sumclient \eta^\clientsup \bracBig{\hat{\delta}_{\pooled-\dml,1}^\clientsup - \hat{\delta}_{\pooled-\dml,0}^\clientsup},\\
&\delta^\clientsup_{\fed-\dml} = \sumclient \hat\ratio^\clientsup_{\fed,1} \hat{\delta}_{\fed-\dml,1}^\clientsup - \sumclient \hat\ratio^\clientsup_{\fed,0} \hat{\delta}_{\fed-\dml,0}^\clientsup,\\
&\text{ with }\hat\ratio^\clientsup_{\collab,1} \propto \hat\datasize_{\collab,1}^\clientsup,\quad \hat\ratio^\clientsup_{\collab,0} \propto \hat\datasize_{\collab,0}^\clientsup.
\end{align*}
Here $\hat\delta^\clientsup_{\pooled-\dml}$ and $\hat\delta^\clientsup_{\fed-\dml}$ are residual versions of the corresponding \ipw~estimators, changing all $\outcome$ to $\outcome - \modeloutcome\paren{\covariates}$ in the formula. We only present the formula for the $\hat\delta_1$'s and relegate $\hat\delta_0$'s to the appendix.
\begin{align*}
&
\hat{\delta}_{\pooled-\dml,1}^\clientsup = \frac{1}{\hat\datasize_{\pooled,1}^\clientsup}\sum_{\indiv\in\dataset^\clientsup} 
\frac{\treatment_\indiv \brac{\outcome_\indiv - \modeltreated\paren{\covariates_\indiv}}}{\modelprop^\clienttrtsup\paren{\covariates_\indiv}},\\
&
\hat{\delta}_{\fed-\dml,1}^\clientsup = \frac{1}{\hat\datasize_{\collab,1}^\clientsup}\sum_{\indiv\in\dataset^\clientsup} 
\frac{\treatment_\indiv \brac{\outcome_\indiv - \modeltreated\paren{\covariates_\indiv}}}{\sumtrtweights}.\\
\end{align*}
The proposed estimator computes the difference in mean of outcome models only in $\datasettarget$ and the correction terms only in $\dataset^\clientsup$'s. Though being decoupled, it preserves the robustness of the \dml~estimator. We summarize its properties in Theorem \ref{thm:dml}.

\begin{theorem}\label{thm:dml}
    Suppose that
    \begin{enumerate}
        \item The estimated models $\hat\modeltreated$, $\hat\modelcontrol$ and $\hat\modelprop$ are independent\footnote{We could achieve independence by using sampling splitting, see \citet{chernozhukov_doubledebiased_2018} for more detailed discussion.} with $\datasettarget$ and $\dataset^\clientsup$'s.
        \item They have convergence rates
    \begin{align}
    \E\brac{\ltwo{\hat\modeltreated - \modeltreated}}&, \E\brac{\ltwo{\hat\modelcontrol - \modeltreated}} = O(1/\datasize^{-\xi_\modeloutcome}),\label{eqn:om-rates}\\
    \text{ and }\E\brac{\ltwo{\hat\modelprop - \modelprop}} &= O(1/\datasize^{-\xi_\modelprop}),\label{eqn:ps-rates}
    \end{align}
    with $\xi_\modeloutcome \xi_\modelprop > 1/2$.
        \item The models $\hat\modelprop$, $\hat\modeloutcome$, $\modelprop$, and $\modeloutcome$ are bounded.
    \end{enumerate}
     Further supposing that $\datasize^\targetsup/\datasizeall \converge{} \tarsrcratio$, we have that 
    \begin{equation}
        \sqrt{\datasize}\paren{\hat\tau_{\collab-\dr} - \tau }\converge{d} \normal\paren{0,v_{\fed-\dr}^2},
    \end{equation}
    with
    \begin{align*}
        v_{\fed-\dr}^2 = &~ \tarsrcratio^\inv \E\bracBig{\brac{\modeltreated\paren{\covariates} - \modelcontrol\paren{\covariates}}^2} -\tarsrcratio^\inv \tau^2 \\
        &~+ \E\bracBig{\frac{\paren{\outcome_\groupt - \modeltreated\paren{\covariates}}^2}{\P\paren{\treatment\paren{\select}=1\mid\covariates}} + \frac{\paren{\outcome_\groupc - \modelcontrol\paren{\covariates}}^2}{\P\paren{\treatment\paren{\select}=0\mid\covariates}}}.
    \end{align*}
\end{theorem}
The assumptions in Theorem \ref{thm:dml} are standard in the literature \citep{chernozhukov_doubledebiased_2018,athey_policy_2020}. If we use the K-NN density ratio estimation \citep{lin_estimation_2021}, we get that $\xi_\modelprop = 2/(2+\dimension)$. Therefore, taking any outcome model with $\xi_\modeloutcome \geq 1/2 - 2/(2+\dimension)$ would guarantee the asymptotic normality of $\hat\tau_{\fed-\dml}$.

\subsection{Estimation of outcome models}
It's worth noting the convergence rates in Equation \eqref{eqn:om-rates} are taking average over the target population. 
To achieve low excess risk in the target population, we adopt the domain adaptation part from orthogonal statistical learning \cite{foster_orthogonal_2020}. Consider the loss function re-weighted through inverse propensity scores:
\begin{align}
L\paren{\modeltreated;\set{\clientdataset^\clientsup}_{\client\in\clientset}} =&~ 
\sumclient  L^\clientsup\paren{\modeltreated;\clientdataset^\clientsup}\nonumber\\
\text{ with } L^\clientsup =&~\sum_{\indiv\in\clientdataset^\clientsup}\frac{\treatment_\indiv \loss\paren{\outcome_\indiv, \modeltreated\paren{\covariates_\indiv}}}{\sumesttrtweights}.\label{eqn:loss-domain-adaptation}
\end{align}
We want to compare it with training directly on the target distribution, i.e., using loss function $\Tilde{L}$
\begin{equation}\label{eqn:loss-target}
\Tilde{L}\paren{\modeltreated;\dataset} = \sum_{\indiv=1}^\datasize \loss\paren{\outcome_\indiv(1), \modeltreated\paren{\covariates_\indiv}}.
\end{equation}
\begin{theorem}\label{thm:orthogonal}
Suppose that 
\begin{enumerate}
    \item The estimated propensity score model $\hat\modelprop\paren{\covariates}$ satisfies Equation \eqref{eqn:ps-rates}.
    \item Using loss function \eqref{eqn:loss-target}, $\hat\modeltreated$ satisfies Equation \eqref{eqn:ps-rates}.
\end{enumerate}
Then, using loss function \eqref{eqn:loss-domain-adaptation}, we have that
\begin{equation}
\E\brac{\ltwo{\hat{\modeltreated}\paren{\covariates} - \modeltreated\paren{\covariates}}} \leq O(1/N^{-\xi_\modeloutcome}) + O(1/N^{-4\xi_\modelprop}).
\end{equation}
\end{theorem}

\subsection{Federated Learning Algorithm}
The estimation of the outcome model requires federated learning. We could optimize the loss function by using FedAvg \cite{li_convergence_2020} or SCAFFOLD \cite{karimireddy_scaffold_2020}. We present the process, including computing $\hat\tau_{\dml}$ in Algorithm \ref{alg:dml}.
\begin{algorithm}
\caption{\fed-\dml~Algorithm}
\label{alg:dml}
\begin{algorithmic}[1]
\REQUIRE $K$ datasets $\set{\clientdataset^\clientsup}_{\client\in\clientset}$ and $\datasettarget$.
\STATE (Locally) estimate $\hat\modelprop^\clientzsup\paren{\covariates}$.
\WHILE{not converged}
\STATE Train model $\modeltreated$ and $\modelcontrol$ using the FedAvg Algorithm with Loss function in \eqref{eqn:loss-domain-adaptation}.
\ENDWHILE
\STATE (Locally) update $\outcome_\indiv$'s by $\outcome_\indiv\rightarrow \outcome_\indiv - \modeloutcome_{\treatment_\indiv}\paren{\covariates_\indiv}$.
\STATE Use Algorithm \ref{alg:fed} to get $\sumclient \delta^\clientsup_{\fed-\dml}$. 
\STATE Construct the \fed-\dml~estimator using Equation \eqref{eqn:def of dml}.
\end{algorithmic}
\end{algorithm}

As a result, using Algorithm \ref{alg:dml}, if we combine Theorems \ref{thm:dml} and \ref{thm:orthogonal}, we could get that $\hat\tau_{\fed-\dml}$ is asymptotic normal given that $\xi_\modeloutcome \xi_\modelprop < 1/2$ and $\xi_\modelprop^5 < 1/2$. Using Proposition \ref{prop:est-acc}, it suffices to utilize the K-NN density ratio estimation method with $\dimension\leq 8$ and find an outcome model with $\xi_\modeloutcome \geq 1/2 - 2/(2+\dimension)$. This avoids the problem of the misspecification of the exponential tilting model.

It is worth noting that our discussion of \dml~is is from the point of view of learning theory. If we adopt the classical double robustness framework, when the outcome model is correctly specified, there's no need to adjust the distribution of the covariates. The \dml~estimator is asymptotically normal even when the propensity score model completely fails. We would demonstrate its robustness in the simulation.

%% file: Sections/experiments.tex
\section{Experiments}
\subsection{Synthetic Dataset}
We conduct the experiment using synthetic dataset. 
Echoing the discussion in Section \ref{sec:setup}, to show that the \textit{sampling}-\textit{selecting} procedure is not necessarily to truly happen, we fix sample sizes and generate the covariates using different distributions. Consider three source datasets, with $\datasize^{(\client)} = 1000,~2000,~3000$. The target dataset contains $\datasize^\targetsup = 10000$ data points.   In specific, we generate the target distribution through $\covariates \sim\normal\paren{\mu^\targetsup,\sigma^2 \bI_3}$ with $\mu^\targetsup = -0.1$ and $\sigma=2$. 

In the source dataset, we fix the treatment assignment mechanism and take the true propensity score as 
\[\P\paren{\treatment^\clientsup = 1\mid \covariates^\clientsup} = 1/\brac{1+\exp\paren{[1.2; 0.3; -1.2]^\top\covariates^\clientsup}}.\]
Take the true potential outcomes as 
\begin{align*}
   &~~Y(1) = [1.2; 1.8; 1.4]^\top \covariates^\clientsup\quad\\
   \text{and}&~~
Y(0) = [0.6; 0.7; 0.6]^\top \covariates^\clientsup. 
\end{align*}
We also choose normal distribution for source datasets. Suppose that $\covariates^\clientsup \sim \normal\paren{\mu^\clientsup,\sigma^2}$, with $\sigma=2$. We use the mean $KL-$divergence between source datasets to the target dataset as a measure for the heterogeneity across sites, which is given by 
\[
d_{\text{KL}}\paren{\datasettarget,\set{\dataset^\clientsup}_{\client\in[3]}} = \sum_{\client=1}^3 \frac{1}{2\sigma^2} \paren{\mu^\clientsup - \mu}^2.
\]
We increase $d_{\text{KL}}$ from 0 to 4. Fixing each $d_{\text{KL}}$, we choose $\mu^\clientsup$ uniformly and randomly assign negative sign to one of them. In the estimation process, we use the exponential tilting model for density ratio estimation and the linear model for outcome regression. We calculate the mean squared error (MSE) of \poolipw,~\collabipw,~and \pooled-\dml, and \fed-\dml~through 2000 Monte Carlo Simulations, with four replications of different $\set{\mu^\clientsup}$'s. 
Figure \ref{fig:kl-mse} shows the $d_{\text{KL}}-$MSE curve. We mark the $\ipw$ estimators in each single site with dotted line. Although outperforming each individual sites, the \poolipw~estimator still suffers from the increasing of heterogeneity. In contrast, both \collabipw~and \dml~remain stable when heterogeneity increases.

We further demonstrate the robustness of the \dml~estimator with four combinations of specifications of propensity score and outcome models. We relegate the detail of how to construct mis-specified model to the appendix. 
Figure \ref{fig:2x2grid} shows the 95\% C.I. of the \poolipw, \collabipw, \pooled-\dml, and \fed-\dml~estimators. We choose the case with the mean $KL$-distance being $3$. In all cases, \collabipw~estimator has tighter confidence intervals. 
When propensity score model is misspecified, both \poolipw~and \collabipw~fail due to incorrect weighting. In contrast, \dml~estimators remain consistent as long as outcome model is correct. When both models get misspecified, there is no hope to obtain consitent result.

\begin{figure}[t!]
    \centering

    \begin{subfigure}[t]{0.45\textwidth}
    \caption{Synthetic dataset}
    \label{fig:2x2grid}
    \includegraphics[width=\linewidth]{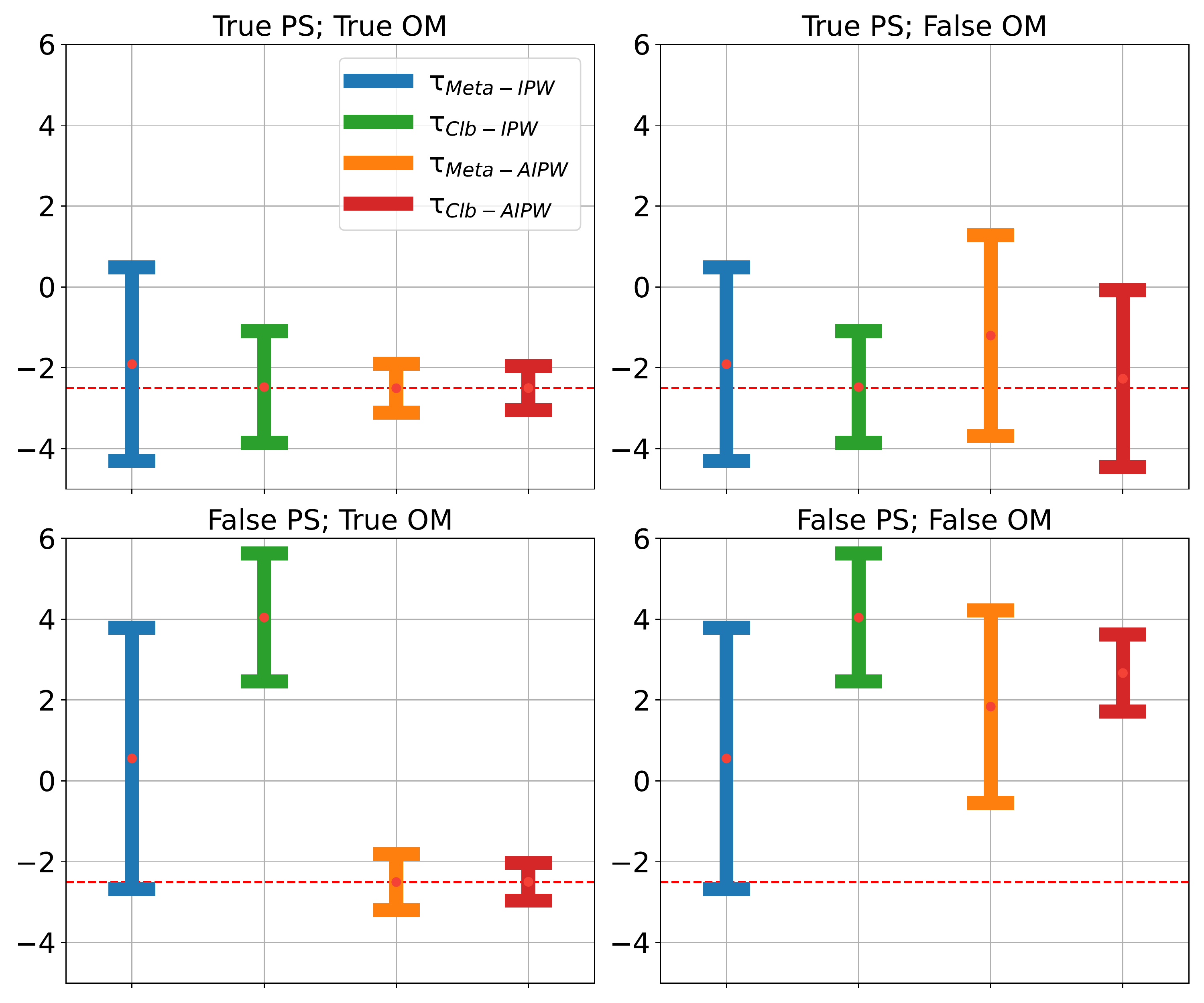}
    \end{subfigure}
\quad\quad
    \begin{subfigure}[t]{0.45\textwidth}
    \caption{Real dataset}
    \label{fig:application}
    \includegraphics[width = \linewidth]{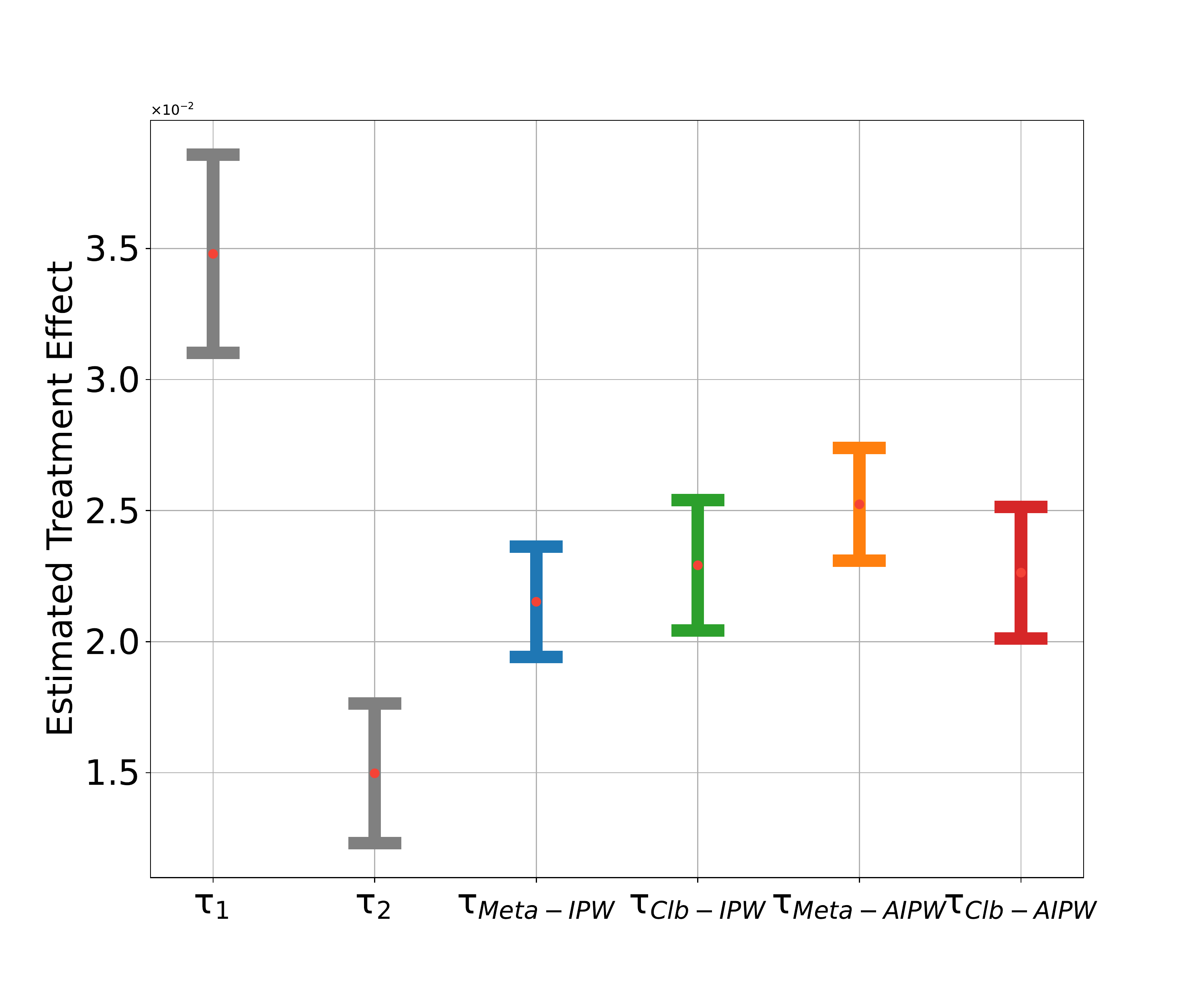}
    \end{subfigure}
\caption{The 95\% confidence intervals for synthetic dataset and the real dataset. The red dots mark the true effect size. In Figure \ref{fig:2x2grid}, \collabipw~shows smaller variance than \poolipw~under all scenarios. The \dml~estimator remains consistent when either of the PS or OM model is correctly specified. In Figure \ref{fig:application}, $\tau_1$ denotes the estimated causal effect in \citet{pennycook_fighting_2020}, and $\tau_2$ denotes \citet{roozenbeek_how_2021}. We find that \poolipw, \collabipw, \pooled-\dml, and \fed-\dml~estimators have similar performance, with \poolipw~ and \pooled-\dml~showing slightly larger effect sizes.}
\label{fig:CIs}
\end{figure}

\subsection{Real world application}
We present a real-world application of our method. Our data comes from two studies about preventing sharing fake news during COVID-19. \citet{roozenbeek_how_2021} replicates the experiment of \citet{pennycook_fighting_2020} to study the effect of a nudge intervention on preventing the sharing of fake news. Both of the two studies sample participants according to U.S. census through online platforms. The outcome is measured by the difference of sharing intentions between true and false headlines about COVID-19 (truth discernment score). 
They find that a simple accuracy reminder could increase the truth discernment score ($\hat\tau = 0.034,$ $p < 0.001$). Using the same design and analysis procedures, \citet{roozenbeek_how_2021} replicates their findings, though with a less significant effect size ($\hat\tau = 0.015,$ $p \approx 0.017$).

Although two studies both try to sample from the target distribution and their heterogeneity is well-controlled, as suggested by \citet{jin_diagnosing_2023}, we still use exponential tilting method to adjust the covariates shift. We adjust the distribution for the mean and variance of the Cognitive Reflection Test (CRT) score, the scientific knowledge quiz score, the Medical Maximizer-Minimizer Scale (MMS), distribution of self-reported political leanings, gender, and age. Figure \ref{fig:application} presents the 95\% C.I.s for the two datasets and three estimators. Due to that the two datasets are close, we find close results. But \collabipw~and \dml~show slightly larger effect size, matching the  conclusion of the original study.
\begin{figure}

\end{figure}

%% file: Sections/discussion.tex
\section{Conclusion}
In this work, we propose a collaborative inverse propensity score estimator that is suitable for heterogeneous data. Along the way, we utilize the \textit{sampling}-\textit{selecting} framework to describe the heterogeneity across sites. We show that the \collabipw~estimator outperforms Meta-analysis-based estimator both in theory and in simulation.
To account for the difficulty of density estimation, we borrow ideas from AIPW and orthogonal statistical learning literature, and provide the necessary convergence rates for nuisance models. As a future direction, it is worth while to explore the communication-efficient method for the optimal weighting of propensity score models.

%% file: Sections/appendix-thm-proof.tex
\section{Proofs}
\subsection{Preliminaries}
\begin{definition}
Given i.i.d. weights $w_\indiv$ and outcomes $\outcome_\indiv$, take their weighted sum 
$
\hat{\sumsite} = \sumn w_\indiv \outcome_\indiv.
$
We call an estimator is ``\hajek" type if it uses $\paren{\sumn w_\indiv}^\inv$ to normalize, and ``Horvitz-Thompson" (\HT) type if it uses $\paren{n \E\brac{w}}^\inv$, i.e.,
\[
\hat\mu_{\hajek} = \frac{1}{\sumn w_\indiv} \hat{\sumsite}\quad \mu_{\HT} = \frac{1}{n\E\brac{w}} \hat{\sumsite}.
\]
\end{definition}
We begin with relating the asymptotic behaviour of \hajek-type \ipw~estimator with the \HT-type. In specific, we have that
\begin{lemma}\label{lem:hajek-to-ht}
The ``Hajek"-type weighted mean estimator is asymptotically equivalent to the centralized ``Horvitz-Thompson"-type weighted mean estimator
\begin{equation}
\hat\mu_\HT = \mu + \frac{1}{n\E\brac{w}} \sumn w_\indiv \paren{\outcome_\indiv - \mu},
\end{equation}
i.e., we have that
\[
\sqrt{n}\paren{\hat\mu_\hajek - \hat\mu_\HT} = o_P(1).
\]
\end{lemma}
\begin{proof}
We subtract $\mu$ from $\hat\mu_\hajek$ and get that
\begin{align*}
\sqrt{n}\paren{\hat\mu_\hajek - \mu} & = \frac{\sqrt{n}}{\sumn w_\indiv} \sumn w_\indiv \paren{\outcome_\indiv - \mu}\\
& = \frac{1}{\sumn w_\indiv / n} \dsqn \sumn w_\indiv \paren{\outcome_\indiv - \mu} \\
& = \frac{1}{\E\brac{w}} \dsqn \sumn w_\indiv\paren{\outcome_\indiv - \mu} + o_P(1)\\
& = \sqrt{n}\paren{\hat\mu_\HT - \mu} + o_P(1).
\end{align*}
The second to the third line is by combining the fact that $\sumn w_\indiv / n = \E\brac{w} + o_P(1)$ and $\sumn w_\indiv\paren{\outcome_\indiv - \mu}/\sqrt{n} = O_P(1)$, through law of large numbers and CLT.
\end{proof}
\subsection{Proof of Proposition \ref{prop:asymp-pooled}}
We first define several useful intermediate values. We use $\hat{\sumsite}$ to denote un-normalized $\ipw$ summations and $\hat{\datasize}$ to denote the estimated data sizes.
\begin{align*}
&\hat\sumsite_\pooled^\clientsup = \sum_{\indiv\in\dataset^\clientsup} \frac{ \treatment_\indiv\outcome_\indiv}{\modelprop^\selecttrtsup\paren{\covariates}} - \frac{\paren{1-\treatment_\indiv}\outcome_\indiv}{\modelprop^\selectctrsup\paren{\covariates}},\quad
\hat\sumsite_{\pooled,1}^\clientsup = \sum_{\indiv\in\dataset^\clientsup} \frac{ \treatment_\indiv\outcome_\indiv}{\modelprop^\selecttrtsup\paren{\covariates}},\quad\text{and}\quad
\hat\sumsite_{\pooled,0}^\clientsup = \sum_{\indiv\in\dataset^\clientsup} \frac{ \paren{1-\treatment_\indiv}\outcome_\indiv}{\modelprop^\selectctrsup\paren{\covariates}}.
\end{align*}
In the main paper, we use that
\begin{align*}
\hat\mu_{\pooled,1}^\clientsup = \frac{1}{\hat\datasize^\clientsup_{\fed,1}} \hat\sumsite_{\pooled,1}^\clientsup \quad
\hat\mu_{\pooled,0}^\clientsup = \frac{1}{\hat\datasize^\clientsup_{\fed,0}} \hat\sumsite_{\pooled,0}^\clientsup.
\end{align*}
\begin{proof}
We first re-write $\hat\sumsite_\pooled^\clientsup$ as
\[
\hat\sumsite_\pooled^\clientsup = \sum_{\indiv=1}^\datasize \frac{\indic{\select_\indiv = \paren{\client,\groupt}} \outcome_\indiv}{\modelprop^\selecttrtsup\paren{\covariates}} - \frac{\indic{\select = \paren{\client,0}}\outcome_\indiv}{\modelprop^\selectctrsup\paren{\covariates}}.
\]
Use Lemma \ref{lem:hajek-to-ht}, we only need to consider
\begin{align*}
\hat\tau_{\pooled-\HT}^\clientsup = \frac{1}{\datasize} \setBig{\frac{\indic{\select_\indiv = \paren{\client,\groupt}} \paren{\outcome_\indiv - \mu_1}}{\modelprop^\selecttrtsup\paren{\covariates}} - \frac{\indic{\select = \paren{\client,0}}\paren{\outcome_\indiv - \mu_0}}{\modelprop^\selectctrsup\paren{\covariates}}} + \tau 
\end{align*}
Note that
\begin{align*}
&~\E\setBig{\frac{\indic{\select = \paren{\client,\groupt}} \paren{\outcome - \mu_1}}{\modelprop^\selecttrtsup\paren{\covariates}} - \frac{\indic{\select = \paren{\client,0}}\paren{\outcome - \mu_0}}{\modelprop^\selectctrsup\paren{\covariates}}}\\ = &~ \E\bracBig{\E\bracBig{\frac{\P\set{\select = \paren{\client,\groupt}\mid\covariates} \paren{\outcome_\groupt - \mu_1}}{\modelprop^\selecttrtsup\paren{\covariates}} - \frac{\P\set{\select = \paren{\client,0}\mid\covariates}\paren{\outcome_\groupc - \mu_0}}{\modelprop^\selectctrsup\paren{\covariates}}\mid \covariates}}\\
= &~ \E\brac{\outcome_\groupt - \mu_\groupt - \paren{\outcome_\groupc - \mu_\groupc}} = 0.
\end{align*}
We also have that
\begin{align*}
&~\Var\setBig{\frac{\indic{\select = \paren{\client,\groupt}} \paren{\outcome - \mu_1}}{\modelprop^\selecttrtsup\paren{\covariates}} - \frac{\indic{\select = \paren{\client,0}}\paren{\outcome - \mu_0}}{\modelprop^\selectctrsup\paren{\covariates}}}\\ = &~ \E\bracBig{\bracBig{\frac{\indic{\select = \paren{\client,\groupt}} \paren{\outcome_\groupt - \mu_1}}{\modelprop^\selecttrtsup\paren{\covariates}} - \frac{\indic{\select = \paren{\client,0}}\paren{\outcome_\groupc - \mu_0}}{\modelprop^\selectctrsup\paren{\covariates}}}^2}\\
= &~ \E\bracBig{\frac{\indic{\select = \paren{\client,\groupt}} \paren{\outcome_\groupt - \mu_1}^2}{\modelprop^\selecttrtsup\paren{\covariates}^2} + \frac{\indic{\select = \paren{\client,0}}\paren{\outcome_\groupc - \mu_0}^2}{\modelprop^\selectctrsup\paren{\covariates}^2}}\\
= &~ \E\bracBig{\frac{\indic{\select = \paren{\client,\groupt}} \paren{\outcome_\groupt - \mu_1}^2}{\modelprop^\selecttrtsup\paren{\covariates}^2} + \frac{\indic{\select = \paren{\client,0}}\paren{\outcome_\groupc - \mu_0}^2}{\modelprop^\selectctrsup\paren{\covariates}^2}}\\
= &~ \E\bracBig{\frac{\paren{\outcome_\groupt - \mu_1}^2}{\modelprop^\selecttrtsup\paren{\covariates}} + \frac{\paren{\outcome_\groupc - \mu_0}^2}{\modelprop^\selectctrsup\paren{\covariates}}}.
\end{align*}
Therefore, using CLT, we get that
\begin{equation}
\sqrt{\datasize} \paren{\hat\tau^\clientsup - \tau} \converge{d} \normal\paren{0,\paren{v_\pooled^\clientsup}^2},
\end{equation}
with
\begin{equation}
\paren{v_\pooled^\clientsup}^2 = \frac{1}{\datasize}\E\bracBig{\frac{\paren{\outcome_\groupt - \mu_1}^2}{\modelprop^\selecttrtsup\paren{\covariates}} + \frac{\paren{\outcome_\groupc - \mu_0}^2}{\modelprop^\selectctrsup\paren{\covariates}}}.
\end{equation}
Therefore, we have that 
\begin{align*}
\sqrt{\datasize} \paren{\hat\tau_\pooled - \tau} & =
\sumclient\bracBig{ \eta^\clientsup \sqrt{\datasize}\paren{\hat\tau^\clientsup_\pooled - \tau}} \converge{d} \normal\parenBig{0,\sumclient \frac{\paren{\eta^\clientsup}^2}{\datasize}\E\bracBig{\frac{\paren{\outcome_\groupt - \mu_1}^2}{\modelprop^\selecttrtsup\paren{\covariates}} + \frac{\paren{\outcome_\groupc - \mu_0}^2}{\modelprop^\selectctrsup\paren{\covariates}}}},
\end{align*}
with
\begin{align*}
v_\pooled^2 & = \sumclient \frac{\paren{\eta^\clientsup}^2\paren{v_\pooled^\clientsup}^2}{\datasize}\\
& \geq \frac{1}{\datasize \sumclient \E\bracBig{\frac{\paren{\outcome_\groupt - \mu_1}^2}{\modelprop^\selecttrtsup\paren{\covariates}} + \frac{\paren{\outcome_\groupc - \mu_0}^2}{\modelprop^\selectctrsup\paren{\covariates}}}^\inv},
\end{align*}
where the equality holds if and only if $\eta^\clientsup \propto \paren{v_\pooled^\clientsup}^\inv$.

\end{proof}

\subsection{Proof of Theorem \ref{thm:fed-hajek}}
We first provide the entire formula for \collabipw~estimator. We define
\begin{align*}
&\hat\sumsite_{\fed,1}^\clientsup = \sum_{\indiv\in\dataset^\clientsup} \frac{ \treatment_\indiv\outcome_\indiv}{\sumtrtweights}
\quad
\hat\sumsite_{\fed,0}^\clientsup = \sum_{\indiv\in\dataset^\clientsup} \frac{ \paren{1-\treatment_\indiv}\outcome_\indiv}{\sumctrweights}\\
&\hat{\datasize}_{\fed,1}^\clientsup = \sum_{\indiv\in\dataset^\clientsup} \frac{ \treatment_\indiv\outcome_\indiv}{\sumtrtweights}\quad
\hat{\datasize}_{\fed,1}^\clientsup = \sum_{\indiv\in\dataset^\clientsup} \frac{ \treatment_\indiv\outcome_\indiv}{\sumtrtweights}.
\end{align*}
Then, we have that
\begin{align*}
\hat\tau_{\fed} = \frac{\sumclient \hat\sumsite_{\fed,1}^\clientsup}{\sumclient \hat{\datasize}_{\fed,1}^\clientsup} - \frac{\sumclient \hat\sumsite_{\fed,0}^\clientsup}{\sumclient \hat{\datasize}_{\fed,0}^\clientsup},
\end{align*}
where in the main paper, we use that
\begin{align*}
\hat\mu_{\fed,1} = \frac{1}{\hat\datasize^\clientsup_{\fed,1}} \hat{\sumsite}_{\fed,1}^\clientsup,\quad
\text{and}\quad
\hat\mu_{\fed,0} = \frac{1}{\hat\datasize^\clientsup_{\fed,0}} \hat{\sumsite}_{\fed,0}^\clientsup.
\end{align*}
\begin{proof}
We rewrite the formula as
\begin{equation}
\hat\sumsite^\clientsup_{\fed,\groupt} = \sum_{\indiv=1}^\datasize \frac{\indic{\select_\indiv = \paren{\client,\groupt}}\outcome_\indiv}{\sumtrtweights}
,\text{ and } 
\hat\sumsite^\clientsup_{\fed,\groupc} = \sum_{\indiv=1}^\datasize \frac{\indic{\select_\indiv = \paren{\client,\groupc}}\outcome_\indiv}{\sumctrweights}.
\end{equation}
As a result, we have that
\begin{align*}
\sumclient \hat\sumsite_{\collab,\groupt} & = \sum_{\indiv=1}^\datasize \sumclient \frac{\indic{\select_\indiv = \paren{\client,\groupt}}\outcome_\indiv}{\sumtrtweights} = \sum_{\indiv=1}^\datasize \frac{\indic{\treatment\paren{\select_\indiv} = \groupt} \outcome_\indiv}{\P\paren{\treatment\paren{\select_\indiv} = \groupt\mid\covariates_\indiv}},\\
\sumclient \hat\sumsite_{\collab,\groupc} & = \sum_{\indiv=1}^\datasize \sumclient \frac{\indic{\select_\indiv = \paren{\client,\groupc}}\outcome_\indiv}{\sumtrtweights} = \sum_{\indiv=1}^\datasize \frac{\indic{\treatment\paren{\select_\indiv} = \groupc} \outcome_\indiv}{\P\paren{\treatment\paren{\select} = \groupc\mid\covariates_\indiv}}.
\end{align*}
Similarly, we get
\begin{align*}
\hat\datasize_{\fed,\groupt} = \sum_{\indiv=1}^\datasize \sumclient \frac{\indic{\select_\indiv = \paren{\client,\groupt}}}{\sumtrtweights} = \sum_{\indiv=1}^\datasize \frac{\indic{\treatment\paren{\select_\indiv} = \groupt}}{\P\paren{\treatment\paren{\select_\indiv} = \groupt\mid\covariates_\indiv}}\\
\hat\datasize_{\fed,\groupc} = \sum_{\indiv=1}^\datasize \sumclient \frac{\indic{\select_\indiv = \paren{\client,\groupc}}}{\sumctrweights} = \sum_{\indiv=1}^\datasize \frac{\indic{\treatment\paren{\select_\indiv} = \groupc}}{\P\paren{\treatment\paren{\select_\indiv} = \groupc\mid\covariates_\indiv}}.
\end{align*}
As a result, $\hat\datasize_{\fed,1}^\inv \hat\sumsite_{\fed,\groupt} - \hat\datasize_{\fed,0}^\inv \hat\sumsite_{\fed,\groupc}$ takes the form of \hajek~type \ipw~estimator. Therefore, we could use Lemma \ref{lem:hajek-to-ht} and get the corresponding \HT-type estimator. Since we have that
\begin{align*}
\E\bracBig{\frac{\indic{\treatment\paren{\select} = \groupt}}{\P\paren{\treatment\paren{\select} = \groupt\mid\covariates}}} = \E\bracBig{\frac{\P\brac{{\treatment\paren{\select} = \groupt}\mid\covariates}}{\P\paren{\treatment\paren{\select} = \groupt\mid\covariates}}} = 1.
\end{align*}
Same result holds for the control group. The \HT~estimators are
\begin{align*}
\paren{\hat\mu_{\fed,\groupt,\HT} - \tau} = \frac{1}{\datasize}\sum_{\indiv=1}^\datasize 
\bracBig{\frac{\indic{\treatment\paren{\select_\indiv} = \groupt}\paren{\outcome_\indiv - \mu_\groupt}}{\P\paren{\treatment\paren{\select_\indiv} = \groupt\mid\covariates_\indiv}} - \frac{\indic{\treatment\paren{\select_\indiv} = \groupc}\paren{\outcome_\indiv - \mu_\groupc}}{\P\paren{\treatment\paren{\select_\indiv} = \groupc\mid\covariates_\indiv}} }
\end{align*}

Using central limit theorem, since we have that
\begin{align*}
&~\E\bracBig{\frac{\indic{\treatment\paren{\select} = \groupt}\paren{\outcome - \mu_\groupt}}{\P\paren{\treatment\paren{\select} = \groupt\mid\covariates}} - \frac{\indic{\treatment\paren{\select} = \groupc}\paren{\outcome - \mu_\groupc}}{\P\paren{\treatment\paren{\select} = \groupc\mid\covariates}}}\\
= &~ \E\bracBig{\frac{\P\paren{\treatment\paren{\select} = \groupt \mid \covariates} \E\brac{\outcome_\groupt - \mu_\groupt\mid\covariates}}{\P\paren{\treatment\paren{\select} = \groupt\mid\covariates}} - \frac{\P\paren{\treatment\paren{\select} = \groupc \mid \covariates} \E\brac{\outcome_\groupc - \mu_\groupc\mid\covariates}}{\P\paren{\treatment\paren{\select} = \groupc\mid\covariates}}}\\
= &~ \E\bracBig{  \E\brac{\outcome_\groupt - \mu_\groupt - \outcome_\groupc + \mu_\groupc\mid\covariates}}\\
= &~ 0.
\end{align*}
and
\begin{align*}
&~\Var\bracBig{\frac{\indic{\treatment\paren{\select} = \groupt}\paren{\outcome - \mu_\groupt}}{\P\paren{\treatment\paren{\select} = \groupt\mid\covariates}} - \frac{\indic{\treatment\paren{\select} = \groupc}\paren{\outcome - \mu_\groupc}}{\P\paren{\treatment\paren{\select} = \groupc\mid\covariates}}}\\
= &~ \E\parenBig{\bracBig{\frac{\indic{\treatment\paren{\select} = \groupt}\paren{\outcome - \mu_\groupt}}{\P\paren{\treatment\paren{\select} = \groupt\mid\covariates}} - \frac{\indic{\treatment\paren{\select} = \groupc}\paren{\outcome - \mu_\groupc}}{\P\paren{\treatment\paren{\select} = \groupc\mid\covariates}}}^2}\\
= &~ \E\parenBig{\frac{\P\paren{\treatment\paren{\select} = \groupt\mid \covariates}\E\brac{\paren{\outcome_\groupt - \mu_\groupt}^2\mid\covariates}}{\P\paren{\treatment\paren{\select} = \groupt\mid\covariates}^2} + \frac{\P\paren{\treatment\paren{\select} = \groupc\mid \covariates}\E\brac{\paren{\outcome_\groupc - \mu_\groupc}^2\mid\covariates}}{\P\paren{\treatment\paren{\select} = \groupc\mid\covariates}^2}}\\
= &~ \E\parenBig{\frac{{\paren{\outcome_\groupt - \mu_\groupt}^2}}{\P\paren{\treatment\paren{\select} = \groupt\mid\covariates}} + \frac{\paren{\outcome_\groupc - \mu_\groupc}^2}{\P\paren{\treatment\paren{\select} = \groupc\mid\covariates}}}.
\end{align*}
Use that $\datasizeall/\datasize \converge{} \P\paren{\select\neq\emptyset}$.
We get that
\begin{equation}
\sqrt{\datasizeall} \parenBig{\hat\tau_\fed - \tau} \converge{d} \normal(0,v_\fed^2),
\end{equation}
with
\begin{equation}
v^2_\fed = \P\paren{{\select}\neq\emptyset}\E\parenBig{\frac{{\paren{\outcome_\groupt - \mu_\groupt}^2}}{\P\paren{\treatment\paren{\select} = \groupt\mid\covariates}} + \frac{\paren{\outcome_\groupc - \mu_\groupc}^2}{\P\paren{\treatment\paren{\select} = \groupc\mid\covariates}}}.
\end{equation}
To compare $v_\fed^2$ and $v_\pooled^2$, we first prove Lemma \ref{lem:concave}.
\begin{lemma}\label{lem:concave}
The function $f(t_1,\ldots,t_\clientnum) = \paren{t_1^\inv + \ldots + t_\clientnum^\inv}^\inv$ with $t_\indiv > 0,~\indiv=1,\ldots,\clientnum$ is concave.
\end{lemma}
\begin{proof}
We directly prove it by showing that its hessian matrix is negative semi-definite. Denoting $\nabla^2 f = \set{H_{kj}}_{1\leq k,j\leq \clientnum}$, we have that
\begin{equation}
H_{kj} = \left\{
\begin{matrix}
\frac{2t_k^{-4}}{\paren{\sum_{\clientp=1}^\clientnum t_\clientp^\inv}^3} - \frac{2 t_k^{-3}}{\paren{\sum_{r=1}^\clientnum t_r^\inv}^2} & \text{if } k=j\\
\frac{2t_k^{-2} t_j^{-2}}{\paren{\sum_{\clientp=1}^\clientnum t_\clientp^\inv}^3} & \text{if } k\neq j.
\end{matrix}
\right.
\end{equation}
By taking out the common factor we get that
\begin{equation}
\frac{1}{2} \parenBig{\sum_{r=1}^\clientnum t_r^\inv}^3 \nabla^2 f\paren{t_1,\ldots,t_\clientnum} = \left(
\begin{matrix}
t_1^{-2} \\ \vdots \\ t_\clientnum^{-2}
\end{matrix}
\right) 
\left(
\begin{matrix}
t_1^{-2} & \ldots & t_\clientnum^{-2}
\end{matrix}
\right) - 
\paren{\sum_{r=1}^\clientnum t_r^\inv} 
\left(
\begin{matrix}
t_1^{-3} & & &\\
& t_2^{-3} & &\\
& & \ddots & \\
& & & t_\clientnum^{-3}
\end{matrix}
\right).
\end{equation}
The second term is negative definite.
The first term only gets one non-zero eigenvalue, with the corresponding eigenvector $v = \paren{t_1^{-2},\ldots,t_\clientnum^{-2}}$. We only need to verify that $v^\top \nabla^2 f v \leq 0$. We have that
\begin{align*}
\frac{1}{2} \parenBig{\sum_{r=1}^\clientnum t_r^\inv}^3 v^\top \nabla^2 f\paren{t_1,\ldots,t_\clientnum} v^\top & = \parenBig{\sumclient t_\client^{-4}}^2 - \paren{\sumclient t_\client^\inv}\paren{\sumclient t_\client^{-7}}\\
& = 2 \sum_{k<j} t_k^{-4}t_j^{-4} - \sum_{k<j} \parenBig{t_k^\inv t_j^{-7} +  t_k^{-7} t_j^\inv}
& \leq 0,
\end{align*}
where the last line is by using the AM-GM inequality and getting that $t_k^\inv t_j^{-7} + t_j^\inv t_k^{-7} \geq 2 t_k^{-4} t_j^{-4}$. This shows that $\nabla^2 f$ is negative semi-definite, which means that $f$ is concave.
\end{proof}
We use Jensen inequality and get that
\begin{align*}
v^2_\pooled & = \frac{2}{\datasize \sumclient \setBig{\E\bracBig{\frac{\paren{\outcome_\groupt - \mu_1}^2}{2\modelprop^\selecttrtsup\paren{\covariates}}} + \E\bracBig{ \frac{\paren{\outcome_\groupc - \mu_0}^2}{2\modelprop^\selectctrsup\paren{\covariates}}}}^\inv} \\
& \geq \frac{1}{\datasize \sumclient \setBig{\E\bracBig{\frac{\paren{\outcome_\groupt - \mu_1}^2}{\modelprop^\selecttrtsup\paren{\covariates}}}}^\inv} + \frac{1}{\datasize \sumclient \setBig{\E\bracBig{ \frac{\paren{\outcome_\groupc - \mu_0}^2}{\modelprop^\selectctrsup\paren{\covariates}}}}^\inv} \\
& \geq \E\bracBig{\frac{1}{\datasize \sumclient \setBig{\frac{\paren{\outcome_\groupt - \mu_1}^2}{\modelprop^\selecttrtsup\paren{\covariates}}}^\inv}} + \E\bracBig{\frac{1}{\datasize \sumclient  \setBig{\frac{\paren{\outcome_\groupc - \mu_0}^2}{\modelprop^\selectctrsup\paren{\covariates}}}^\inv}} \\
& = \E\bracBig{\frac{\paren{\outcome_\groupt - \mu_1}^2}{\datasize \sumclient \modelprop^\selecttrtsup\paren{\covariates}}} + \E\bracBig{\frac{\paren{\outcome_\groupc - \mu_0}^2}{\datasize \sumclient  \modelprop^\selectctrsup\paren{\covariates}}} \\
& = v_{\fed}^2,
\end{align*}
where we use Jensen twice at the second and the third lines. 
\end{proof}

\subsection{Proof of Theorem \ref{thm:balancing}}
\begin{proof}
The proof relies on the definition of independence and Assumption \ref{ass:unconfoundedness}. Using $p_{y,s}$ to denote the joint density function for $\paren{\outcome(1),\outcome(0)}$ and $\select$, and $p_y$, $\P_s$ as their marginal distributions, we have that
\begin{align*}
p_{y,s}\set{\paren{\outcome(1),\outcome(0)} , \select\mid\covariates} & = p_{y}\set{\paren{\outcome(1),\outcome(0)}\mid\covariates}\P_s\set{ \select\mid\covariates}.
\end{align*}
Take expectation conditional on $\P\paren{\select = \paren{\client,\treatmentdata}\mid\covariates} = \modelprop^\clientzsup\paren{\covariates}$, and use the tower property of cognitional expectation, we get that, for the L.H.S., 
\begin{align*}
\E\bracBig{p_{y,s}\set{\paren{\outcome(1),\outcome(0)} , \select\mid\covariates}~\Big|~ \modelprop^\clientzsup\paren{\covariates}} = p_{y,s}\setBig{\paren{\outcome(1),\outcome(0)} , \select\mid\modelprop^\clientzsup\paren{\covariates}};
\end{align*}
for the R.H.S., 
\begin{align*}
\E\bracBig{p_{y}\set{\paren{\outcome(1),\outcome(0)} \mid\covariates} \P_s\set{ \select\mid\covariates} ~\Big|~ \modelprop^\clientzsup\paren{\covariates}} & = \E\bracBig{p_{y}\set{\paren{\outcome(1),\outcome(0)} \mid\covariates}  ~\Big|~ \modelprop^\clientzsup\paren{\covariates}} \modelprop^\clientzsup\paren{\covariates}\\
& = p_{y}\setBig{\paren{\outcome(1),\outcome(0)} \mid\modelprop^\clientzsup\paren{\covariates}}\modelprop^\clientzsup\paren{\covariates}\\
& = p_{y}\setBig{\paren{\outcome(1),\outcome(0)} \mid\modelprop^\clientzsup\paren{\covariates}}\P_s\setBig{\select\mid\modelprop^\clientzsup\paren{\covariates}}.
\end{align*}
This shows that 
\[
\paren{\outcome(1),\outcome(0)} \indep \select \mid \modelprop^\clientzsup\paren{\covariates}.
\]
For the second part, similarly, using tower property, we have that
\begin{align*}
\E\bracBig{p_{y,z}\set{\paren{\outcome(1),\outcome(0)} , \treatment(\select) \mid\covariates}~\Big|~ \P\brac{\treatment(\select)\mid\covariates}} = p_{y,s}\setBig{\paren{\outcome(1),\outcome(0)} , \select\mid\P\brac{\treatment(\select)\mid\covariates}};
\end{align*}
for the R.H.S., 
\begin{align*}
\E\bracBig{p_{y}\set{\paren{\outcome(1),\outcome(0)} \mid\covariates} \P_z\set{ \treatment(\select)\mid\covariates} ~\Big|~ \P_z\set{ \treatment(\select)\mid\covariates}} & = \E\bracBig{p_{y}\set{\paren{\outcome(1),\outcome(0)} \mid\covariates}  ~\Big|~ \modelprop^\clientzsup\paren{\covariates}} \modelprop^\clientzsup\paren{\covariates}\\
& = p_{y}\setBig{\paren{\outcome(1),\outcome(0)} \mid\P_z\set{ \treatment(\select)\mid\covariates}}\P_z\set{ \treatment(\select)\mid\covariates}\\
& = p_{y}\setBig{\paren{\outcome(1),\outcome(0)} \mid\P_z\set{ \treatment(\select)\mid\covariates}}
\P\setBig{\treatment(\select)\mid\P_z\set{ \treatment(\select)\mid\covariates}}.
\end{align*}
This shows that 
\[
\paren{\outcome(1),\outcome(0)} \indep \treatment(\select) \mid \P\set{ \treatment(\select)\mid\covariates}.
\]

\end{proof}

\subsection{Proof of Proposition \ref{prop:identification} and \ref{prop:est-acc}}
We only provide the proof for Proposition \ref{prop:identification}. For the proof of Proposition \ref{prop:est-acc}, see \citet{zhao_entropy_2017} and \citet{lin_estimation_2021}.
\begin{proof}
Suppose that another distribution ${\paren{\select^\prime,\covariates^\prime,\outcome_\groupt^\prime,\outcome_\groupc^\prime}}$ generates the same observed distribution $p(\covariatesdata),$ $p(\covariatesdata\mid \select = \paren{\client,\groupt})$, and $p(\covariatesdata\mid \select = \paren{\client,\groupc})$ for all $\client$. Using Bayes' theorem,
\begin{align*}
\P\paren{\select^\prime = \paren{\client,\groupt}\mid \covariates^\prime} & = \frac{p\paren{\covariatesdata^\prime \mid \select^\prime = \paren{\client,\groupt}}\P\paren{\select^\prime = \paren{\client,\groupt}}}{p\paren{\covariatesdata^\prime}}\\
& = \frac{p\paren{\covariatesdata \mid \select = \paren{\client,\groupt}}\P\paren{\select^\prime = \paren{\client,\groupt}}}{p\paren{\covariatesdata}}\\
& = \P\paren{\select = \paren{\client,\groupt}\mid\covariates}\frac{\P\paren{\select^\prime = \paren{\client,\groupt}}}{\P\paren{\select = \paren{\client,\groupt}}}.
\end{align*}
This shows that $\modelprop^\clientzsup\paren{\covariates}$ is identifiable up to a constant, but $\P\paren{\select = \emptyset}$ is not identifiable.
\end{proof}

\subsection{Proof of Theorem \ref{thm:dml}}
We first give the formulas for $\hat\mu_{\groupc}^\clientsup$:
\begin{align*}
&
\hat{\delta}_{\pooled-\dml,0}^\clientsup = \frac{1}{\hat\datasize_{\pooled,0}^\clientsup}\sum_{\indiv\in\dataset^\clientsup} 
\frac{\paren{\treatment_\indiv} \brac{\outcome_\indiv - \modeltreated\paren{\covariates_\indiv}}}{\modelprop^\clienttrtsup\paren{\covariates_\indiv}},\quad
\hat{\delta}_{\fed-\dml,1}^\clientsup = \frac{1}{\hat\datasize_{\collab,1}^\clientsup}\sum_{\indiv\in\dataset^\clientsup} 
\frac{\paren{1 - \treatment_\indiv} \brac{\outcome_\indiv - \modeltreated\paren{\covariates_\indiv}}}{\sumtrtweights}.
\end{align*}
\begin{proof}
Consider the following estimator using the true outcome and propensity score models:
\begin{equation}
\tilde\tau_{\collab-\dr} = \frac{1}{\datasize^\targetsup}\sum_{\indiv\in\datasettarget} \bracBig{\modeltreated\paren{\covariates_\indiv^\targetsup} - \modelcontrol\paren{\covariates_\indiv^\targetsup}} + \frac{1}{\hat\datasize_{\collab}} \sumclient  \hat\datasize_\fed^\clientsup \tilde\taudr_{\fed}^\clientsup,
\end{equation}
with 
\begin{equation}
\tilde\taudr^\clientsup_{\fed-\textsc{dr}} = \frac{1}{\hat\datasize_\fed^\clientsup}\sum_{\indiv\in\clientdataset^\clientsup}\bracBig{\frac{\treatment\paren{\select_\indiv} \paren{\outcome_\indiv - \modeltreated\paren{\covariates_\indiv}}}{\sumtrtweights} - \frac{\paren{1-\treatment\paren{\select_\indiv}} \paren{\outcome_\indiv - \modelcontrol\paren{\covariates_\indiv}}}{\sumctrweights}},
\end{equation}
where $\modeltreated,$ $\modelcontrol$, and $\modelprop$ are true models. We first prove Lemma \ref{lem:coupling}.
\begin{lemma}\label{lem:coupling}
We have that
    \begin{equation}
\sqrt{\datasizeall}\paren{\hat{\tau}_{\fed-\dr} - \tilde{\tau}_{\fed-\dr}} \converge{d} 0.
\end{equation}
\end{lemma}
\begin{proof}
Similar to the proof of Theorem \ref{thm:fed-hajek}, using Lemma \ref{lem:hajek-to-ht}, we have that
\begin{equation}\label{eqn:dr-hajek-to-ht}
\sqrt{\datasize}\paren{\hat\tau_{\fed-\dr} - \hat{\tau}_{\fed-\dr-\HT}} \converge{d} 0\quad\text{and}\quad \sqrt{\datasize}\paren{\tilde\tau_{\fed-\dr} - \tilde{\tau}_{\fed-\dr-\HT}} \converge{d} 0,
\end{equation}
with 
\begin{align*}
&~\hat\tau_{\fed-\dr-\HT}\\
 = &~ \frac{1}{\datasize^\targetsup} \sum_{\indiv \in \datasettarget}\bracBig{\hat\modeltreated\paren{\covariates} - \hat\modelcontrol\paren{\covariates}}  + \frac{1}{\datasize}\sum_{\indiv=1}^\datasize \bracBig{\frac{\treatment\paren{\select_\indiv} \paren{\outcome_\indiv - \hat\modeltreated\paren{\covariates_\indiv}}}{\sumesttrtweights} - \frac{\paren{1-\treatment\paren{\select_\indiv}} \paren{\outcome_\indiv - \hat\modelcontrol\paren{\covariates_\indiv}}}{\sumestctrweights}},
 \end{align*}
and
 \begin{align*}
&~\tilde\tau_{\fed-\dr-\HT}\\
 = &~ \frac{1}{\datasize^\targetsup} \sum_{\indiv \in \datasettarget}\bracBig{\modeltreated\paren{\covariates} - \modelcontrol\paren{\covariates}}  + \frac{1}{\datasize}\sum_{\indiv=1}^\datasize\bracBig{\frac{\treatment\paren{\select_\indiv}\paren{\outcome_\indiv - \modeltreated\paren{\covariates}}}{\P\paren{\treatment\paren{\select_\indiv} = \groupt\mid\covariates_\indiv}} - \frac{\paren{\groupt - \treatment\paren{\select_\indiv}}\paren{\outcome_\indiv - \modelcontrol\paren{\covariates}}}{\P\paren{\treatment\paren{\select_\indiv} = \groupc\mid\covariates_\indiv}}}.
\end{align*}
We decompose $\hat{\tau}_{\fed-\dr-\HT} - \tilde{\tau}_{\fed-\dr-\HT}$
\begin{align*}
&~\hat{\tau}_{\fed-\dr-\HT} - \tilde{\tau}_{\fed-\dr-\HT} \\
= &~ \frac{1}{\datasize} \sum_{\indiv=1}^\datasize\bracBig{{\frac{\treatment\paren{\select_\indiv}}{\P\paren{\treatment\paren{\select_\indiv} = \groupt\mid\covariates_\indiv}} - \frac{\treatment\paren{\select_\indiv}}{\sumesttrtweights}}}\paren{\outcome_\indiv - \hat\modeltreated\paren{\covariates_\indiv}} \\
&~ - \frac{1}{\datasize} \sum_{\indiv=1}^\datasize\bracBig{{\frac{1-\treatment\paren{\select_\indiv}}{\P\paren{\treatment\paren{\select_\indiv} = \groupc\mid\covariates_\indiv}} - \frac{1-\treatment\paren{\select_\indiv}}{\sumestctrweights}}}\paren{\outcome_\indiv - \hat\modelcontrol\paren{\covariates_\indiv}}\\
&~ + \frac{1}{\datasize} \sum_{\indiv=1}^\datasize \setBig{\bracBig{\frac{\treatment\paren{\select_\indiv}}{\P\paren{\treatment\paren{\select_\indiv} = \groupt\mid\covariates_\indiv}}}\bracBig{\modeltreated\paren{\covariates_\indiv} - \hat\modeltreated\paren{\covariates_\indiv}} - \E\bracBig{\modeltreated\paren{\covariates_\indiv} - \hat\modeltreated\paren{\covariates_\indiv}}}\\
&~ - \frac{1}{\datasize} \sum_{\indiv=1}^\datasize   \setBig{\bracBig{\frac{1-\treatment\paren{\select_\indiv}}{\P\paren{\treatment\paren{\select_\indiv} = \groupc\mid\covariates_\indiv}}}\bracBig{\modelcontrol\paren{\covariates_\indiv} - \hat\modelcontrol\paren{\covariates_\indiv}} - \E\bracBig{\modelcontrol\paren{\covariates_\indiv} - \hat\modelcontrol\paren{\covariates_\indiv}}}\\
&~ + \frac{1}{\datasize^\targetsup} \sum_{\indiv\in\datasettarget} \setBig{\bracBig{\hat\modeltreated\paren{\covariates_\indiv^\targetsup} - \modeltreated\paren{\covariates_\indiv^\targetsup}} - \E\bracBig{\hat\modeltreated\paren{\covariates} - \modeltreated\paren{\covariates}}}\\
&~+ \frac{1}{\datasize^\targetsup} \sum_{\indiv\in\datasettarget} \setBig{\bracBig{\hat\modelcontrol\paren{\covariates_\indiv^\targetsup} - \modelcontrol\paren{\covariates_\indiv^\targetsup}} - \E\bracBig{\hat\modelcontrol\paren{\covariates} - \modelcontrol\paren{\covariates}}}.
\end{align*}
Denote the above terms as $\Delta_1,\ldots,\Delta_6$. We bound each of them. 
\begin{align*}
\abs{\Delta_1} \leq &~ \sqrt{\frac{1}{\datasize}\sum_{\indiv=1}^\datasize \frac{\treatment\paren{\select_\indiv}^2 \bracBig{\P\paren{\treatment\paren{\select_\indiv} = \groupt\mid\covariates_\indiv} - \sumesttrtweights}^2}{\P\paren{\treatment\paren{\select_\indiv} = \groupt\mid\covariates_\indiv}^2\brac{\sumesttrtweights}^2}}\sqrt{\frac{1}{\datasize}\sum_{\indiv=1}^\datasize \bracBig{\outcome_\indiv\paren{1} - \modeltreated\paren{\covariates_\indiv}}^2} \\
\leq &~ \sqrt{\frac{\pslower^{-2}}{\datasize}\sum_{\indiv=1}^\datasize  \bracBig{\P\paren{\treatment\paren{\select_\indiv} = \groupt\mid\covariates_\indiv} - \sumesttrtweights}^2}\sqrt{\frac{1}{\datasize}\sum_{\indiv=1}^\datasize \bracBig{\outcome_\indiv\paren{1} - \modeltreated\paren{\covariates_\indiv}}^2},
\end{align*}
By taking expectation and applying Jensen inequality, we get that
\begin{align*}
\E\brac{\sqrt{\datasize}\abs{\Delta_1}} \leq &~ \sqrt{\datasize}\sqrt{ \pslower^{-2} \E\bracBig{\P\paren{\treatment\paren{\select} = 1\mid\covariates} - \sumrvtrtweights}^2} \sqrt{\E\setBig{\bracBig{\outcome\paren{1} - \modeltreated\paren{\covariates}}^2}} \\
\leq &~ {\datasize^{1/2 -\ratepower_{\modeloutcome} - \ratepower_{\modelprop}}} \converge{} 0,
\end{align*}
as $\datasize\to\infty$. This shows that $\sqrt{\datasize}\abs{\Delta_1} \converge{P}0$. We could prove that $\sqrt{\datasize}\Delta_2\converge{P}0$ with the same manner.
Using the Bernstein Inequality for bounded random variables \citep{vershynin_high-dimensional_2018}, we have that
\begin{align*}
\P\setBig{\sqrt{\datasize}\abs{\Delta_3} \geq t/\pslower} \leq &~ \P\setBig{\sqrt{\datasize}\absBig{\sum_{\indiv=1}^\datasize\bracBig{\modeltreated\paren{\covariates_\indiv} - \hat\modeltreated\paren{\covariates_\indiv}} - \E\bracBig{\modeltreated\paren{\covariates_\indiv} - \hat\modeltreated\paren{\covariates_\indiv}}}\geq t} \\
\leq &~ 2\exp\parenBig{-\frac{t^2/2}{\sum_{\indiv=1}^\datasize\Var\brac{\hat\modeltreated\paren{\covariates_\indiv} - \modeltreated\paren{\covariates_\indiv}}/\datasize + \boundy t/\paren{3\sqrt{\datasize}}}} \\
\leq &~ 2\exp\parenBig{-\frac{t^2/2}{\E\set{\brac{\hat\modeltreated\paren{\covariates} - \modeltreated\paren{\covariates}}^2} + \boundy t/\paren{3\sqrt{\datasize}}}} \\
\leq &~ 2\exp\parenBig{-\frac{t^2/2}{\datasize^{-2\xi_\modeloutcome} + \datasize^{-1/2}\boundy t/{3}}}  \converge{} 0,
\end{align*}
for any $t > 0$ and $\datasize\converge{}\infty$. This proves that $\sqrt{\datasize}\Delta_3 \converge{P} 0$. We could prove that $\sqrt{\datasize}\Delta_4\converge{P}0$ with the same manner. At last, for $\Delta_5$, we have that
\begin{align*}
\P\setBig{\sqrt{\datasize}\abs{\Delta_3} \geq t} \leq &~ 2\exp\parenBig{-\frac{t^2/2}{\sum_{\indiv\in\datasettarget}\Var\brac{\hat\modeltreated\paren{\covariates_\indiv} - \modeltreated\paren{\covariates_\indiv}}/\datasize^\targetsup + \boundy t/\paren{3\sqrt{\datasize^\targetsup}}}} \\
\leq &~ 2\exp\parenBig{-\frac{t^2/2}{\E\set{\brac{\hat\modeltreated\paren{\covariates} - \modeltreated\paren{\covariates}}^2} + \boundy t/\paren{3\sqrt{\datasize^\targetsup}}}} \\
\leq &~ 2\exp\parenBig{-\frac{t^2/2}{\paren{\datasize^\targetsup}^{-2\xi_\modeloutcome} + \paren{\datasize^\targetsup}^{-1/2}\boundy t/{3}}}  \converge{} 0,
\end{align*}
for any $t > 0$ and $\datasize \converge{}\infty$. This proves that $\sqrt{\datasize}\Delta_5 \converge{P} 0$. We could prove that $\sqrt{\datasize}\Delta_6\converge{P}0$ with the same manner. Combining $\Delta_1,\ldots,\Delta_6$ with Equation \eqref{eqn:dr-hajek-to-ht} together, we finish the proof of Lemma \ref{lem:coupling}.
\end{proof}
By Lemma \ref{lem:coupling}, we only need to consider $\tilde\tau_{\fed-\dr}$. Using CLT, we have that
\begin{equation}
\sqrt{\datasize}\paren{\tilde\tau_{\fed-\dr} - \tau}\converge{d} \normal\paren{0,v_{\collab-\dr}^2},
\end{equation}
since
\begin{align*}
\E\paren{\tilde{\tau}_{\fed-\dr}} = &~ \E\bracBig{\modeltreated\paren{\covariates^\targetsup} - \modelcontrol\paren{\covariates^\targetsup}} + \E \bracBig{\frac{\treatment\paren{\select}\paren{\outcome_\groupt - \modeltreated\paren{\covariates}}}{\P\paren{\treatment\paren{\select}=1\mid\covariates}} - \frac{\paren{1-\treatment\paren{\select}}\paren{\outcome_\groupc - \modelcontrol\paren{\covariates}}}{\P\paren{\treatment\paren{\select}=0\mid\covariates}}}\\
= &~ \E\brac{\outcome_\groupt - \outcome_\groupc},
\end{align*}
and with
\begin{align*}
v_{\collab-\dr}^2 =&~  \Var\bracBig{\sqrt{\datasize}\paren{\tilde{\tau}_{\fed-\dr} - \tau}} \\
= &~ \frac{\datasize}{\datasize^\targetsup}\Var\bracBig{\modeltreated\paren{\covariates^\targetsup} - \modelcontrol\paren{\covariates^\targetsup}} \\
&~ + \frac{\datasize}{\datasize}\Var \bracBig{\frac{\treatment\paren{\select}\paren{\outcome_\groupt - \modeltreated\paren{\covariates}}}{\P\paren{\treatment\paren{\select}=1\mid\covariates}} - \frac{\paren{1-\treatment\paren{\select}}\paren{\outcome_\groupc - \modelcontrol\paren{\covariates}}}{\P\paren{\treatment\paren{\select}=0\mid\covariates}}}\\
=&~  \tarsrcratio^\inv \E\bracBig{\brac{\modeltreated\paren{\covariates} - \modelcontrol\paren{\covariates}}^2} -\tarsrcratio^\inv \tau^2 + \E\bracBig{\frac{\paren{\outcome_\groupt - \modeltreated\paren{\covariates}}^2}{\P\paren{\treatment\paren{\select}=1\mid\covariates}} + \frac{\paren{\outcome_\groupc - \modelcontrol\paren{\covariates}}^2}{\P\paren{\treatment\paren{\select}=0\mid\covariates}}}.
\end{align*}
This proves Theorem \ref{thm:dml}.

\end{proof}

\subsection{Proof of Theorem \ref{thm:orthogonal}}
It is a direct result from  Appendix B.2 in \citet{foster_orthogonal_2020}.

%% file: Sections/appendix-exp.tex
\section{Experiments}
\subsection{Extra Details}
For the incorrect scenario, using subscript $i$ to denote different dimensions of $\covariates$, we let $\covariates_1^\prime = \covariates_1 \covariates_2$, $\covariates_2^\prime = \covariates_2^2$, and $\covariates_3^\prime = \covariates_3 / \max\set{1,\covariates_1^\prime}$. Using $\covariates^\prime$ as the regressors for misspecified propensity and outcome models.

\subsection{Ablations}
We provide the $KL-$MSE plots with misspecified models in Figure \ref{fig:ablations}. All experiment settings are the same with Figure \ref{fig:kl-mse}, but we perturb the models. We construct false models also with $\covariates^\prime$. The results show the same trend with Figure \ref{fig:kl-mse}. It is worth noting that in Figure \ref{fig:TF}, the \dml~estimator has similar variance with \poolipw~when $KL$ distance is large. We attribute this result to numerical instability, as we find there are occasionally divergent learned parameters due to extreme heterogeneity. The \collabipw~estimator maintains low MSE against heterogeneity.
\begin{figure}[ht]
    \centering
    \begin{subfigure}{0.31\textwidth}
        \centering
        \caption{True PS; False OM}
        \label{fig:TF}
        \includegraphics[width=\linewidth]{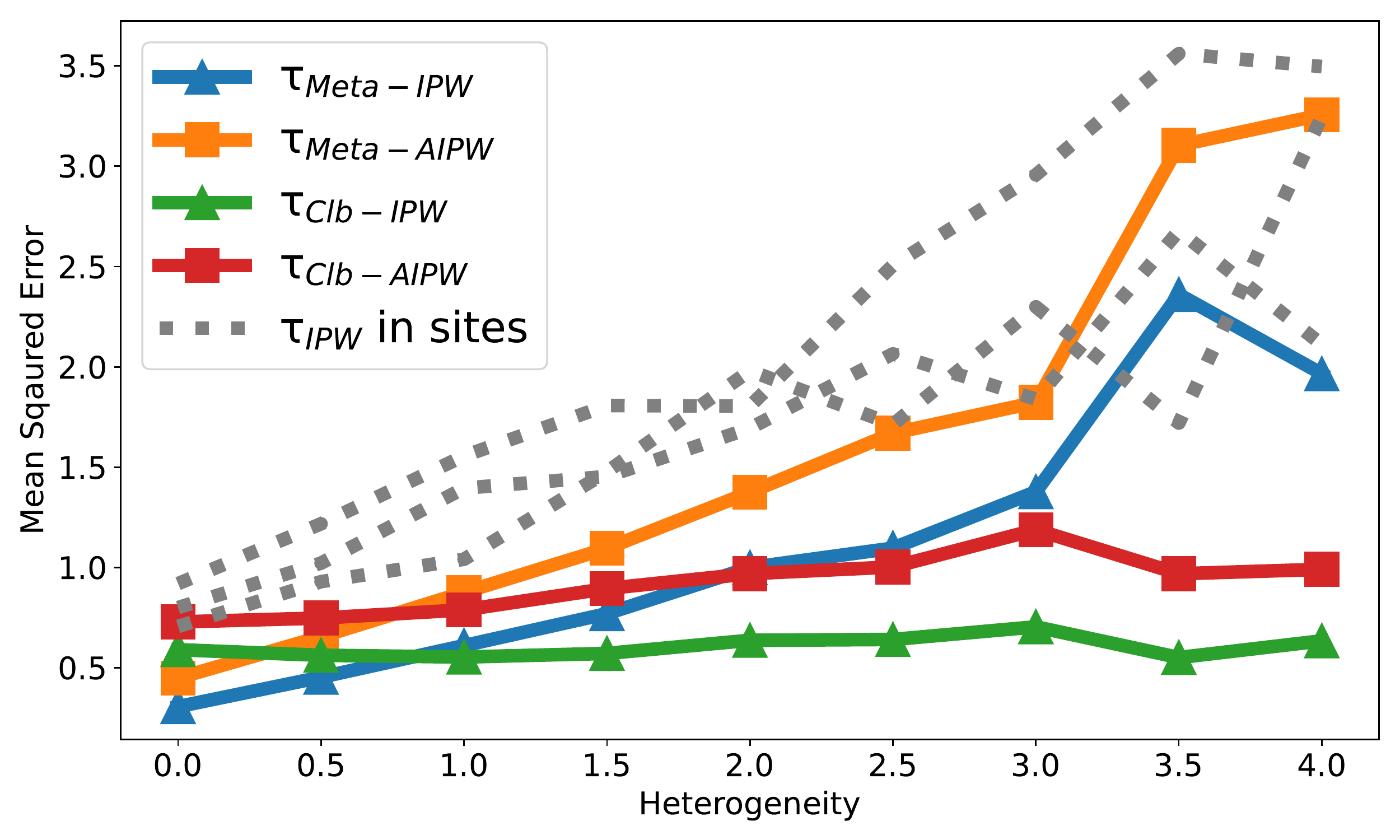}
        \end{subfigure}
    \begin{subfigure}{0.31\textwidth}
        \centering
        \caption{False PS; True OM}
        \label{fig:FT}
        \includegraphics[width=\linewidth]{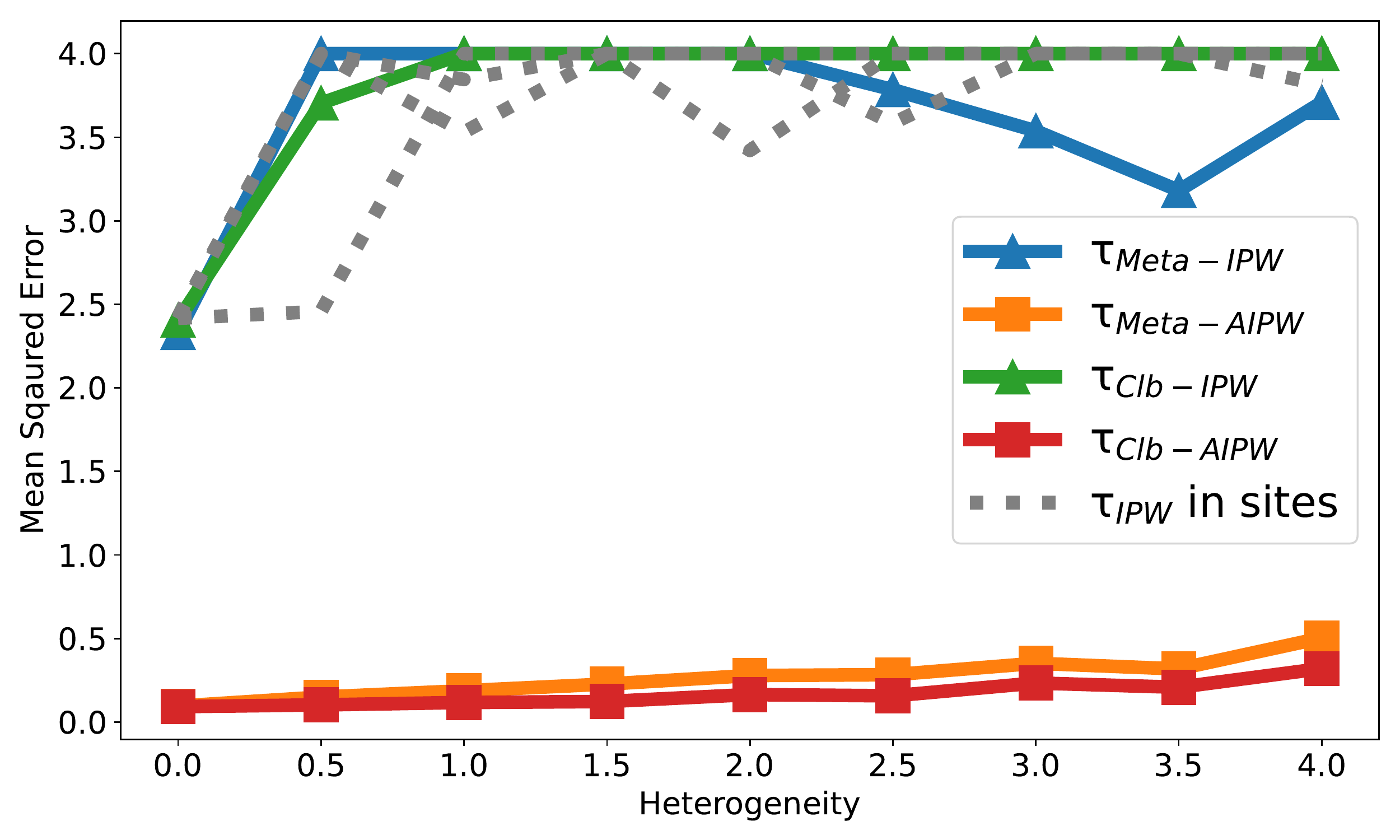}
    \end{subfigure}
    \begin{subfigure}{0.31\textwidth}
        \centering
        \caption{False PS; False OM}
        \label{fig:FF}
        \includegraphics[width=\linewidth]{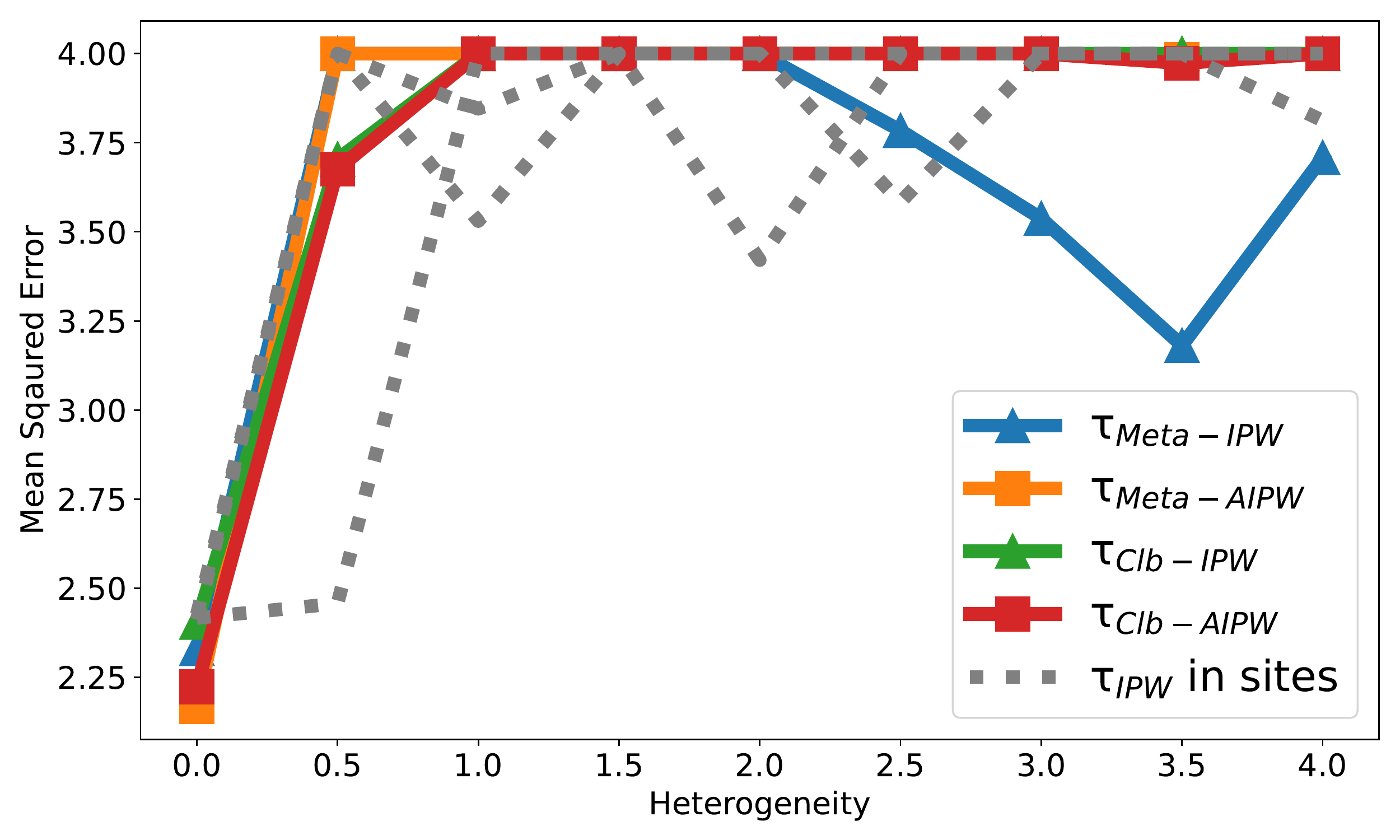}
    \end{subfigure}
    \caption{The mean squared error changing with heterogeneity. We use $\covariates^\prime$ for all misspecified models. When both models fail to fit the data, there's no theoretical guarantee and all estimators have huge mean squared error. The better performance of \poolipw~there is meaningless.}
    \label{fig:ablations}
\end{figure}